\documentclass[11pt, a4paper]{article}
\usepackage{amsmath,amsthm,amscd,amsfonts,amssymb,enumitem}
\usepackage{latexsym}
\usepackage{multirow,ltablex}
\usepackage{caption,subcaption,float}
\usepackage{graphicx,epstopdf,pgfplots,wrapfig}
\usepackage{soul,color}
\usepackage{todonotes,bm}

\usepackage[colorlinks=true, linkcolor=blue, anchorcolor=blue, citecolor=blue, filecolor=blue, urlcolor=blue]{hyperref}

\newtheorem{thm}{Theorem}[section]
\newtheorem{lemma}[thm]{Lemma}

\theoremstyle{definition}
\newtheorem{defin}[thm]{Definition}

\newtheorem{exam}[thm]{Example}
\newtheorem{rem}[thm]{Remark}
\theoremstyle{remark}

\newcommand\numberthis{\addtocounter{equation}{1}\tag{\theequation}}
\newcommand{\vc}[1]{\textup{vec}\big(#1\big)}

\numberwithin{equation}{section}
\newcommand\blfootnote[1]{%
	\begingroup
	\renewcommand\thefootnote{}\footnote{#1}%
	\addtocounter{footnote}{-1}%
	\endgroup
}

\def\NN{\mathbb N}

\def\RR{\mathbb R}

\def\nn{\mathcal N}

\def\ii{\mathcal I}
\def\ll{\mathcal L}
\def\pp{\mathcal P}

\def\W{\bm{W}}

\def\sgs{\sigma_{\textup{small}}}
\def\sgmin{\sigma_{\textup{min}}}

\title{Convergence of gradient based training for linear Graph Neural Networks}
\author{Dhiraj Patel\blfootnote{Email id: patel@cs.rwth-aachen.de, savostianov@cs.rwth-aachen.de, schaub@cs.rwth-aachen.de}, Anton Savostianov, Michael T. Schaub\\
	Computational Network Science, RWTH Aachen University\\ Aachen, Germany
	}

\date{}

\begin{document}
	\maketitle
	
	\begin{abstract}
        Graph Neural Networks (GNNs) are powerful tools for addressing learning problems on graph structures, with a wide range of applications in molecular biology and social networks. However, the theoretical foundations underlying their empirical performance are not well understood. In this article, we examine the convergence of gradient dynamics in the training of linear GNNs. Specifically, we prove that the gradient flow training of a linear GNN with mean squared loss converges to the global minimum at an exponential rate. The convergence rate depends explicitly on the initial weights and the graph shift operator, which we validate on synthetic datasets from well-known graph models and real-world datasets. Furthermore, we discuss the gradient flow that minimizes the total weights at the global minimum. In addition to the gradient flow, we study the convergence of linear GNNs under gradient descent training, an iterative scheme viewed as a discretization of gradient flow.
	\end{abstract}
	
	
	\section{Introduction}
	Deep learning methods have been extensively applied in many applications of the machine learning field, such as computer vision \cite{he2016deep}, speech recognition \cite{dahl2011context}, and natural language processing \cite{raffel2020exploring}. It helps identify patterns from large and complex datasets. Gradient descent is a learning paradigm of deep neural networks, used to train the networks to optimize the prediction error and significantly enhance their accuracy. While deep neural networks have shown exceptional performance in many fields, the mathematical principles underlying their success are not well understood. In this paper, we study the mathematical properties of the optimization method for deep learning performance.
	
	The mathematical study of learning deep networks focuses on the convergence of the gradient descent algorithm, which optimizes the prediction error of the network. Non-linear deep neural networks express complex and non-linear relationships within datasets; however, the optimization problem is challenging due to their non-convex architecture. On the contrary, linear neural networks are not ideal for modeling complex data in many machine learning applications due to their simple structure. However, the non-convex loss in linear deep neural networks makes the optimization problem non-trivial. In recent years, articles \cite{arora2018convergence,arora2018optimization,bah2022learning,yacine2023geometric,kawaguchi2016deep,trager2020pure} have explored important properties of gradient dynamics for linear deep neural networks and studied the optimization problem in that setting.
	
	In order to train deep networks, data are usually collected from Euclidean domains in many applications \cite{aggarwal2018neural}, such as text mining, image classification, speech recognition, and stock price prediction. However, data structures generally have complex relationships and dependencies that can be embedded in non-Euclidean domains such as graphs. Deep learning on graphs can be used to detect communities \cite{schaub2017many}, predict user preferences \cite{wu2022graph}, and identify the properties of molecules \cite{wieder2020compact}. Graph neural networks have been gaining attention for their applications in geometric deep learning.
	
	Graph neural networks are primarily concerned with two types of learning tasks: graph classification tasks \cite{defferrard2016convolutional,kearnes2016molecular} and node regression or classification tasks \cite{kipf2017semisupervised,xu2021optimization}. In the applications, these tasks demonstrate exceptional results. For example, the graph classification task is applied to classify various chemical compounds \cite{srinivasan1994mutagenesis}, and the node regression task is used to perform traffic forecasting \cite{cui2019traffic}. However, the mathematical aspects of GNNs to comprehend their accuracy and limitations are not well studied, even less so than those of classical deep learning methods. This raises the following questions: Under what conditions does the gradient descent training of GNNs converge to the global minimum? How do we control the speed of convergence in training? How do we minimize the total weight at the global minimum? In this article, we study the convergence analysis of the gradient descent training of linear GNN, particularly for the node regression problem.

    \paragraph{Background and related works.} The convergence of the gradient descent method for a convex objective function is well understood \cite{boyd2004convex,nesterov2013introductory}. In fact, all local minima of a convex objective function are global minima. However, a sub-optimal local minimum may exist for a non-convex objective function. As a consequence, finding the global minimum of a general non-convex function by gradient descent is NP-complete \cite{murty1985some}. In order to ensure the convergence of gradient descent to the global optimum, it is necessary to assume some additional conditions on the network. Recent theoretical studies have proven the global optimal solution for gradient-based training of the neural networks with linear activation \cite{arora2018convergence,arora2018optimization,bah2022learning,xu2021optimization}, ReLU activation \cite{awasthi2021convergence,li2017convergence}, smooth activation \cite{chatterjee2022convergence}, and Gaussian inputs \cite{brutzkus2017globally,ge2018learning}.
    
    Several researchers have concentrated their study on the critical points of objective functions to investigate the optimal training loss \cite{ge2015escaping,lee2016gradient}. Lee et al. \cite{lee2016gradient} show that gradient descent with a random initialization converges to the local minimum if the objective function satisfies the strict saddle point property. 
    While a shallow network always satisfies the strict saddle point property, a deep network with more than two hidden layers probably fails to satisfy such a condition \cite{kawaguchi2016deep}.
    On the other hand, articles \cite{arora2018convergence,bah2022learning,chatterjee2022convergence,yacine2023geometric,saxe2014exact} focus on the initialization and the training data for the convergence of the gradient descent method in deep neural networks. In the linear deep neural network, Bartlett et al. \cite{bartlett2018gradient} prove that the gradient descent algorithm with the square loss converges to the global minimum if the initial loss is sufficiently close to the global minimum. Later, Chatterjee \cite{chatterjee2022convergence} generalizes the result for non-linear deep neural networks. The articles \cite{arora2018convergence,bah2022learning} study the convergence analysis of gradient flow under the balanced conditions of the initial weight matrix.
    
    The convergence analysis in the context of graph neural networks has not been explored much.
    Xu et al. \cite{xu2021optimization} consider a linear graph neural network model and show that the positive singular margin of the initial weights is sufficient to ensure the square loss converges to the global minimum. Indeed, the loss surface is devoid of any sub-optimal local minima \cite{laurent2018deep}.
    Awasthi et al. \cite{awasthi2021convergence} generalize this result for shallow GNNs with ReLU activation and examine the convergence of gradient descent training in a probabilistic framework. In particular, with Gaussian initialization, the gradient descent algorithm recovers the realizable data of the GNN architecture with high probability.


    \paragraph{Contributions.} We present a comprehensive study on the convergence of gradient descent training for linear graph neural networks. The key features of the article are as follows.
	\begin{itemize}
		\item In a semi-supervised setting, we show that the mean square loss exponentially converges to its global minimum in a gradient flow training of linear GNNs. The convergence rate is proportional to the singular value of the initial weight matrices and the feature data embedded on graphs. Moreover, we discuss the gradient descent algorithm for linear GNN training.
		
		\item The gradient flow training of linear GNN with respect to square loss converges to the global minimum, provided the initial weight parameters have a positive singular margin \cite{xu2021optimization}. In practice, verifying such a condition for any given weight matrix is challenging. We study the convergence analysis of gradient dynamics in linear GNNs without such assumptions and provide an initialization that is free from the restriction on saddle points, imposed by singular margin assumption, and ensures convergence to the global minimum.
		
		\item The convergence analysis of the gradient dynamics is discussed for deep neural networks, assuming the given data matrix is full rank \cite{arora2018convergence,bah2022learning,chatterjee2022convergence}, and in the context of GNN, the given feature matrix associated with the graph shift matrix is full rank \cite{awasthi2021convergence,xu2021optimization}. To the best of our knowledge, the convergence analysis of gradient-based training in deep neural networks associated with low-rank data has not been studied in the literature. In this article, we show that the convergence of the mean square loss to the global minimum depends on the smallest non-zero singular value of the product of the feature data and graph shift matrix.
		
		\item Under the ``nice" initialization, the gradient-based training converges to the global minimum of the loss surface. However, based on the GNN architecture, the loss surface might have infinitely many global minima. Minimizing the weights at the global minima reduces the computational complexity of the model and compresses the memory storage while the model remains consistent in terms of its accuracy and performance. In this context, we explore the properties of initialization, which leads the training process towards minimizing the loss and optimizing the weights.
	\end{itemize}
	
    \paragraph{Outline.}
	The paper is organized as follows. Linear GNNs are defined in Section \ref{sec:prelim}, which discusses the preliminary results of matrix theory and the mean square loss associated with linear GNN. Section \ref{sec:grad_flow} explores the convergence analysis of gradient flow training for linear GNN, while the optimization of the weights at the global minima of the loss surface is discussed in Section \ref{sec:min_energy}. The analogous results for the convergence of gradient flow in discrete time are addressed in Section \ref{sec:grad_desc}. Finally, in Section \ref{sec:numerical}, we validate the gradient dynamics on synthetic graph datasets and a real-world graph dataset, and we discuss the limitations of our result.
	
	\section{Preliminaries and problem setup}\label{sec:prelim}
	In this section, we recall the notion of graphs and set up the graph neural network model. We then recap some preliminary results from matrix theory that are used later in the paper.

%
%
%
%

	\subsection{Graphs and Graph Neural Networks}
    \paragraph{Graphs and data supported on graphs.}
    A graph is defined as an ordered pair $G=(V,E)$, which consists of a set of nodes $V$ and a set of edges $E\subseteq \{\{u,v\} \mid u,v \in V\}$ describing pairwise relations between the nodes. 
    For simplicity, we identify the node set with the integers from $1$ up to $n$ as follows, i.e., $V=\{1,\dots,n\}$, where $n$ is the number of nodes in the graph.
    This enables us to encode the structure of the graph into an adjacency matrix $A\in\{0,1\}^{n\times n}$ with entries $A_{ij}=1$ if the edge $\{i,j\}$ exists in the edge set of the graph, and $A_{ij}=0$, otherwise.
    We say a matrix $S\in\mathbb{R}^{n\times n}$ is a graph shift operator if it has the same sparsity pattern as $A$ in the off-diagonal terms, i.e., we have $S_{ij}\neq 0 \Leftrightarrow A_{ij}\neq 0$ for all $i\neq j$.

    In the context of machine learning on graphs, we are interested in semi-supervised learning, where each node in the graph has feature data associated with it, and the graph contains both labeled and unlabeled nodes. We are concerned with predicting the labels of the nodes based on the graph structure and node features.
    We assume that each node $i$ is endowed with a feature vector $x_i \in\mathbb{R}^{d_x}$ (describing, e.g., certain attributes of the node), and $X\in \RR^{d_x\times n}$ denotes the feature data matrix with the $i$-th columns corresponding to feature data $x_i$. The set $\ii\subset V$ is the collection of labeled nodes and $y_j\in\mathbb{R}^{d_y}$ (describing quantities we may want to estimate or predict) is the label vector for node $j\in \ii$. The labeled node matrix is denoted by $Y=[y_i]_{i\in \ii}\in \RR^{d_y\times \bar{n}}$, where $\bar{n}$ is the number of labeled nodes.
    
    \paragraph{Graph Neural Networks.}
    
    Graph neural network is a mapping from node feature data to label data, and it operates by transmitting and aggregating messages between graph nodes. Our problem of interest is to train GNN to improve its accuracy in the prediction of label data.
    
    The architecture of deep GNN is composed of several hidden layers, each containing node data derived from the aggregation of the node data from the previous layer. Mathematically, if $X_{\ell}$ denotes the collection of node data on the $\ell$-th layer, then data representation on $(\ell+1)$-th layer is given by
    \begin{align*}
    	Z_{\ell}&=F(X_{\ell})\\
    	X_{\ell+1}&=\varphi_{\ell}(Z_{\ell}),
    \end{align*}
    where $F$ denotes the linear filter equipped with trainable weights, and $\varphi_{\ell}:\RR\to \RR$ is an activation function that acts componentwise. In particular, graph convolution filter is defined by $$F(X)=\sum_{k=1}^{K} W_{k}XS^{k},$$ where $S$ is the graph aggregate matrix, and $W_{k}$'s are trainable weights. GNNs with convolution filters $K\leq 1$ are referred to as message passing networks, since the node representation depends exclusively on its immediate neighbors. In semi-supervised learning, Kipf and Willing \cite{kipf2017semisupervised} discussed the effectiveness of a message passing network with the ReLU activation function. In this article, we explore the mathematical properties of message passing network optimization using linear activation functions. 
    
    \begin{defin}[Linear GNN]
    	Let $G=(V,E)$ be a graph and $S\in \RR^{n\times n}$ denote the corresponding graph shift matrix. The linear GNN with $H$ layers is a map $f:\RR^{d_x\times n} \to \RR^{d_y\times n}$ such that 
    	\begin{equation}
    		\label{eqn:linear_GNN}
    		f(X)=W_{H+1}X_H,\qquad \text{and } \qquad X_\ell=W_\ell X_{\ell-1}S, \quad \text{for } 1\leq \ell\leq H,
    	\end{equation}
    	where $W_1,\dots,W_{H+1}$ represent the weight matrices of the corresponding order during propagation, and $X_0$ denotes the feature matrix $X$. 
    \end{defin}
    
    The output of the linear GNN depends on the collection of weight matrices $\W=\big(W_1,W_2,\dots,W_{H+1}\big)$, and our aim is to improve the accuracy of the output by appropriately choosing the weight matrices. For simplicity, we denote the map for linear GNN by $f(\cdot\,, \W)$, where the trainable weights $W_{\ell}\in \RR^{d_{\ell}\times d_{\ell-1}}$ for $1\leq \ell\leq H+1$, $d_0=d_x$, and $d_{H+1}=d_y$.
    
    The linear GNN of layer $H$ can be viewed as an iterative model that gathers node information from $\hat{X}_{\ell-1}=X_{\ell-1}S$ from its neighbors and encodes it into a node representation $X_{\ell}=W_{\ell}\hat{X}_{\ell-1}$. More specifically, the node representation of a given node extracts contextual information from its $H$-hop neighborhood. 
    The architecture of GNNs relies on the way data is collected from neighboring nodes, namely through different forms of an aggregation matrix. The adjacency matrix of the graph is one of the most popular choices for an aggregation operator. However, the adjacency matrix aggregates the node features of all neighboring nodes, which can lead to optimization difficulty and numerical instability for graphs with large variations in the degrees of nodes. 
    One approach to address this issue is to consider $S=\hat{D}^{-1}(I_n+A)$ \cite{hamilton2017inductive}, where $\hat{D}$ is the degree matrix of the self-loop adjacency matrix $I_n+A$ of the graph $G$. Such aggregator matrix reduces computation costs by sampling a fixed number of neighbors for aggregation at random. 
    The aggregation operator $S=(I_n+A)\hat{D}^{-1}$ is considered in the PageRank algorithm \cite{brin1998pagerank,gasteiger2018predict}, which revises the ranking of a page by computing the weighted sum of the rankings of its linked pages. 
    In a citation graph, information from the most cited papers may be insufficient for categorizing the paper, as multiple articles reference these papers across a wide range of subfields. Kipf and Welling \cite{kipf2017semisupervised} consider the aggregation matrix $S=\hat{D}^{-\frac{1}{2}}(I_n+A)\hat{D}^{-\frac{1}{2}}$ for smoothing the ``graph signal". We refer the reader to \cite{dong2022improving,hamilton2020graph} for a more detailed discussion on the choice of the aggregation matrix.
    
    In order to improve the prediction of the output label data, the linear GNN is trained to minimize a loss function. Various loss functions are considered in GNN training based on different network applications. For instance, cross-entropy loss is considered for the node classification problem \cite{lee2022grafn,zhou2020towards}, Quasi-Wasserstein loss is used for the optimal transportation problem \cite{cheng2023quasi,essid2018quadratically,facca2021fast}, and mean squared loss is applied for the regression task \cite{awasthi2021convergence,pham2022neural}. 
    
    In the semi-supervised node regression problem, a GNN is not only trained to appropriately predict the accessible label data on some subset of the node set but also to efficiently predict for unlabeled nodes. In this article, we study the semi-supervised node regression problem for a linear GNN.	In particular, we discuss the following minimization problem. For a given input data matrix $X\in \RR^{d_x\times n}$ and output label matrix $Y\in \RR^{d_y\times \bar{n}}$, 
    \begin{equation}\label{eqn:trainingloss}
    	\tilde{\ll}_H=\inf_{\W}\ll(\W):=\inf_{\W}\frac{1}{m}\|f(X,\W)_{*\ii}-Y\|_{F}^2.
    \end{equation}
    We assume that the label data is known on some nodes $\ii\subset V$ with cardinality $\bar{n}$. Moreover, $f(X,\W)_{*\ii}$ denotes the predicted label data matrix on labeled nodes, and $m=\bar{n}d_y$.
    
    For any $H\in \NN$, the optimization problem defined in \eqref{eqn:trainingloss} is non-convex. However, the network $f(X,\W)_{*\ii}=W_{[1:H+1]}(XS^H)_{*\ii}$ with the factorization $$W_{[1:H+1]}=W_{H+1}\cdot W_H\cdots W_1$$ can be viewed as an overparameterization of the matrix $W_{[1:H+1]}$. The factorization imposes an addition constraint that the rank of $W_{[1:H+1]}$ is at most $k=\min\{d_x,d_1,d_2,\dots,d_H,d_y\}$. This implies, 
    $$\inf_{\W}\ll(\W)\geq \inf_{W\in \RR^{d_y\times d_x}} \frac{1}{m}\|W(XS^H)_{*\ii}-Y\|_F^2.$$
    
    \paragraph{Assumption.} We consider the linear GNN model with the hidden feature dimensions are in non-increasing order, i.e., $d_1\geq d_2\geq \cdots \geq d_{H+1}.$
    
    For a given linear operator $R:\RR^{d'}\to \RR^{d}$, the map $\pp_{R}:\RR^{d}\to \RR^{d}$ denotes an orthogonal projection onto the column space of $R$ and defined by $\pp_{R}:= R\cdot R^{\dagger},$ where $R^{\dagger}$ denotes the Moore–Penrose inverse of $R$. The projection map is used to estimate the least square solution, which subsequently helps to calculate the lower bound for \eqref{eqn:trainingloss}.
    
    \begin{exam}
    	For $W\in \RR^{d_y\times d_x}$, the loss $\ll_1(W)=\|W(XS^H)_{*\ii}-Y\|_F^2$ is a convex function, and the least square solution to the square loss $\ll_1$ can be explicitly determined. In particular, we have 
    	\begin{align*}
    		\inf_{W\in \RR^{d_y\times d_x}}\|W(XS^H)_{*\ii}-Y\|_F^2
    		&= \inf_{W\in \RR^{d_y\times d_x}}\big\| \big[ (XS^H)_{*\ii}^{\top}\otimes I_{d_y} \big]\vc{W}-\vc{Y} \big\|_2^2\\
    		&= \big\|\pp_{[ (XS^H)_{*\ii}^{\top}\otimes I_{d_y} ]}\vc{Y}-\vc{Y}\big\|_2^2.
    	\end{align*}
    	The first equality follows from a well-known result in matrix theory \cite{zhan2013matrix}, which states that for any $P\in \RR^{d_{P}\times d_{P}'}$, $Q\in \RR^{d_{Q}\times d_{Q}'}$ and $X\in \RR^{d_{P}'\times d_{Q}}$,
    	\begin{equation}
    		\label{eqn:vec}
    		\vc{PXQ}=\big(Q^{\top}\otimes P\big)\vc{X},
    	\end{equation}
    	where $\otimes$ denotes the Kronecker product of matrices.
    \end{exam}
    
    		%
    %
    %
    		%

    \subsection{Preliminaries on matrix theory}
    Graph Neural Networks (GNNs) are trained to minimize the prediction loss $\ll(\W)$. The primary goal of this paper is to investigate the convergence of gradient-based training for linear GNNs to estimate $\tilde{\W}$, the parameter that optimizes the loss $\ll$. To delve deeper into the training of linear GNNs, we first recall some preliminary results from matrix theory.
    
	The following lemmas outline basic properties of full rank matrices, which are used in the main result on gradient flow convergence. While these results are straightforward to derive, we provide complete proofs to ensure our exposition is self-contained.
	
	\begin{lemma}\label{lemma:sgmin_sum}
		Let $P\in \RR^{d\times d'}$ be a full rank matrix. 
        Then, for any matrix $\tilde{P}\in \RR^{d\times d'}$ with Frobenius norm $\|\tilde{P}\|_F<\sgmin(P),$ the matrix $(P+\tilde{P})$ is also full rank. Moreover, the smallest singular value of the sum of the matrices $(P+\tilde{P})$ is bounded below by $\big( \sgmin(P)-\|\tilde{P}\|_F \big)$.
        
        Here, $\sgmin(M)$ denotes the smallest singular value of the matrix $M$. 

	\end{lemma}
    The lemma precisely states that every matrix in an open neighborhood of a full rank matrix is also a full rank. In particular, the set of all full rank matrices of the same order is an open set with respect to the Frobenius norm.
	
	\begin{proof}
		The proof of the lemma easily follows from the Weyl's inequality for the singular value of matrices \cite{johnson1985matrix}. For any $P, \tilde{P}\in \RR^{d\times d'}$, we have
		\begin{align}
			\label{eqn:sgmin_sum}
			\sigma_i(P+\tilde{P})\geq \sigma_i(P)-\|\tilde{P}\|, \qquad \text{and} \qquad \|\tilde{P}\|\leq \|\tilde{P}\|_F,
		\end{align}
		where $\sigma_i(P)$ and $\|\tilde{P}\|$ denote the $i$-th singular value of $P$ and the spectral norm of $\tilde{P}$ respectively.
	\end{proof}
	
	A full rank matrix preserves the structural information of the input data and avoids the loss of dimensionality redundancy. In a linear deep neural network, data are processed with linear transformation at each layer. A low rank transformation at certain layer might lose critical information about the data. The following lemma emphasizes the critical role of full rank matrices in maintaining structural integrity and minimizing information loss during successive linear transformations.
	
	\begin{lemma}\label{lemma:sgmin_product}
		Suppose $P\in \RR^{d\times d'}$ and $Q\in \RR^{d'\times d''}$ are full rank matrices, with $d\geq d'\geq d''$. 
        Then, the product $PQ$ is a full rank matrix and $$\sgmin(PQ)\geq \sgmin(P)\sgmin(Q).$$ 
	\end{lemma}
	
	\begin{proof}
		The fact that $P$ and $Q$ are full rank matrices implies that $\text{rank}(P)=d'$, and $\text{rank}(Q)=d''$. 
        Using Sylvester's rank inequality, we get $PQ$ is full rank as
		\begin{align*}
			d''\geq \text{rank}(PQ)\geq \text{rank}(P)+\text{rank}(Q)-d'=d''.
		\end{align*}
        Let $x\in \RR^{d''}$ be arbitrary non-zero column matrix. Then $x$ is in row space of $Q$, and $Qx$ is in row space of $P$. Hence, 
		\begin{align*}
			\|PQx\|_2\geq \sgmin(P)\|Qx\|_2\geq \sgmin(P)\sgmin(Q)\|x\|_2.
		\end{align*}
		This implies, $\sgmin(PQ)\geq \sgmin(P)\sgmin(Q.)$. 
	\end{proof}
	
	
	\begin{lemma}
		\label{lemma:small_singular}
		Let $R\in \RR^{d\times d'}$ be a real matrix. For every $x\in \RR^{d'}$, we have
		$$\|Rx\|_2\geq \sgs(R)\,\|\pp_{R^{\top}}x\|_2,$$ where $\sgs(R)$ denotes the smallest non-zero singular value of $R$.
	\end{lemma}
	\begin{proof}
		Let $R=U\Sigma V^{\top}$ be the singular value decomposition of $R$, where $U\in \RR^{d\times d}$ and $V\in \RR^{d'\times d'}$ are the orthonormal matrices, and $\Sigma\in \RR^{d\times d'}$ is the rectangular diagonal matrix with singular values of $R$ as the diagonal entries.
		
		Let $x\in \RR^{d'}$ be arbitrary. Then we have
		\begin{align*}
			\|Rx\|_2^2&=\|\Sigma V^{\top}x\|_2^2\\
			&= x^{\top}V\Sigma^{\top}\Sigma V^{\top}x\\
			&\geq \sigma_{k}^2(R)\,x^{\top}V\Sigma^{\top}(\Sigma^{\dagger})^{\top} V^{\top}x\\
			&= \sigma_{k}^2(R)\, \langle \pp_{R^{\top}}x, x \rangle = \sigma_{k}^2(R)\,\|\pp_{R^{\top}}x\|_2^2.
		\end{align*}
		
		This completes the proof.
	\end{proof}
	
	\section{Convergence analysis of gradient flow training}\label{sec:grad_flow}
	
	
	
	
	In this section, we discuss the gradient flow training of linear GNNs and show that the gradient flow (gradient descent with infinitesimal steps) of means square loss $\ll$ converges to the global minimum. The gradient flow training of linear GNNs evolves the weight matrices at time $t$ as follows
	\begin{equation}
		\label{eqn:gradient_flow}
		\frac{dW_{\ell}(t)}{dt}=-\frac{\partial \ll\big(\W(t)\big)}{\partial W_{\ell}}, \qquad 1\leq \ell\leq H+1,
	\end{equation}
	where $W_{\ell}(t)$ represents the trainable parameters at time $t$ with initialization $W_{\ell}(0)$ and $\W(t)=\big(W_{1}(t),W_{2}(t),\dots,W_{H+1}(t)\big)$.
	
	The optimization of gradient flow training for the classical neural network has been studied in \cite{arora2018convergence,bah2022learning,chatterjee2022convergence,yacine2023geometric} and has been generalized to GNNs \cite{awasthi2021convergence,xu2021optimization}. The main result of the paper regarding gradient flow training of linear GNNs is stated as follows.
	
	\begin{thm}\label{thm:gradient_flow}
		Let $G=(V,E)$ be a graph embedded with a feature matrix $X\in \RR^{d_x\times n}$, and a labeled matrix $Y\in \RR^{d_y\times \bar{n}}$. Let us consider a linear GNN with $H$ layers and non-increasing hidden feature dimensions. The gradient flow training \eqref{eqn:gradient_flow} of the linear GNN with the loss function $\ll$ defined in \eqref{eqn:trainingloss} converges to the global minimum under an appropriately chosen initialization.
		
		In particular, suppose the initial weight $W_1(0)$ is the zero matrix and $W_{\ell}(0)$ is full rank for $2\leq \ell\leq H+1$. Then, with the initialization $\W(0)$:
		\begin{align*}
			\ll\big(\bm{W}(T)\big)-\tilde{\ll}_H&\leq e^{-\frac{1}{m}\beta\sgs^2((XS^H)_{*\ii})T}\big( \ll(\W(0))-\tilde{\ll}_H \big),
		\end{align*} 
		provided $\sgmin\big(W_{\ell}(0)\big)$ is sufficiently large for some $2\leq \ell\leq H+1$. Here $\sgs((XS^H)_{*\ii})$ denotes the smallest non-zero singular value of the matrix $(XS^H)_{*\ii}$, $\beta=\frac{1}{4^{H-1}}\prod_{\ell=2}^{H+1} \sgmin^2(W_{\ell}(0))$ and $m=\bar{n}d_y.$
	\end{thm}
	
	The proof of Theorem \ref{thm:gradient_flow} is based on the following local optimization result.
	\begin{thm}[\cite{chatterjee2022convergence}]\label{thm:sourav}
		Let $p$ be a positive integer, and $F:\RR^p\to [0,\infty)$ be a non-negative $C^2$ function. For any $x_0\in \RR^p$, if there exists $r>0$ such that 
		\begin{equation}
			\label{eqn:alpha_condi}
			4F(x_0)<r^2\inf_{x\in B(x_0,r), F(x)\neq 0} \frac{|\nabla F(x)|^2}{F(x)},
		\end{equation}
		then the gradient flow equation $\frac{d}{dt}\phi(t)=-\nabla F\big(\phi(t)\big)$ has a unique solution $\phi:[0,\infty)\to \RR^p$ with the initialization $\phi(0)=x_0$. Here, $B(x_0,r)$ is the closed Euclidean ball of radius $r$ centered at $x_0$.
		
		Moreover, the solution $\phi$ stays in $B(x_0,r)$ for all $t>0$, and converges to some $\tilde{x}\in B(x_0,r)$ where $F(\tilde{x})=0$. In particular, for $\displaystyle \alpha:=\inf_{x\in B(x_0,r), F(x)\neq 0} |\nabla F(x)|^2/F(x)$, and $t\geq 0$, we have $$\|\phi(t)-\tilde{x}\|_2\leq re^{-\alpha t/2}, \qquad \text{and} \qquad F\big(\phi(t)\big)\leq e^{-\alpha t}F(x_0).$$
	\end{thm}
	
	The above theorem characterize the convergence of gradient flow for the twice differentiable non-negative functions. The inequality \eqref{eqn:alpha_condi} ensure the absence of saddle points in the neighborhood of the initialization and the gradient flow always stays in that neighborhood. Moreover, the initial points is appropriately chosen close to the global minimum. In the following, utilizing the idea of the Theorem \ref{thm:sourav}, we prove the convergence result for the linear GNNs.
	
	\begin{proof}[Proof of Theorem \ref{thm:gradient_flow}]
		The idea of the convergence result is as follows.
		\begin{itemize}
			\item First consider the initial weights such that there does not exist any critical points other than the points of global minimum in the neighborhood of initial weights. In other words, if $\W(0)$ denotes the collection of initial weights then $$\inf \bigg\{ \frac{|\nabla L(\W)|^2}{L(\W)}: \sum_{\ell=1}^{H+1} \|W_{\ell}-W_{\ell}(0)\|_F^2\leq r^2 \bigg\}>0,$$ where $|\nabla L(\W)|^2=\sum\limits_{\ell=1}^{H+1} \|\nabla_{W_{\ell}} L(\W)\|_F^2,$ and the loss function $L$ is defined by 
			\begin{equation}
				\label{eqn:L_loss}
				L(\W)=\ll(\W)-\tilde{\ll}_H.
			\end{equation}
			
			\item Second steps of the prove is to ensure that the gradient flow always stay in that neighborhood of the initial weights. In particular, the analogous condition of \eqref{eqn:alpha_condi} for the linear GNNs, i.e., $$4L\big(\W(0)\big)<r^2\inf \bigg\{ \frac{|\nabla L(\W)|^2}{L(\W)}: \sum_{\ell=1}^{H+1} \|W_{\ell}-W_{\ell}(0)\|_F^2\leq r^2 \bigg\}.$$
		\end{itemize}
		
		Let $\W(0)$ be the collection of the initial weights such that $W_{1}(0)\in \RR^{d_1\times d_x}$ is a zero matrix, and $W_{\ell}(0)\in \RR^{d_{\ell}\times d_{\ell-1}}$ is full rank for $2\leq \ell\leq H+1$. 
		
		The lower bound of $\frac{\|\nabla_{\vc{W_1}} L(\W)\|_2^2}{L(\W)}$ in the neighborhood of the initial weights is a lower bound of $\frac{|\nabla L(\W)|^2}{L(\W)}$ as well. Hence, for any collection of weights $\W$, $\nabla_{\vc{W_1}} L(\W)$ is estimated by
		\begin{multline*}
			\nabla_{\vc{W_1}} L(\W)=\left( \frac{\partial L}{\partial \textup{vec}(\hat{Y})}\frac{\partial \textup{vec}(\hat{Y})}{\partial \vc{(X_H)_{*\ii}}}\frac{\partial \vc{(X_H)_{*\ii}}}{\partial \vc{X_{H-1}}}\frac{\partial \vc{X_{H-1}}}{\partial \vc{X_{H-2}}}\cdots\frac{\partial \vc{X_{1}}}{\partial \vc{W_1}} \right)^{\top}
		\end{multline*}
		where
		\begin{align*}
			\hat{Y}&=f(X,\W)_{*\ii}=W_{H+1}(X_H)_{*\ii}\in \RR^{d_y\times \bar{n}},\\
			(X_H)_{*\ii}&= W_{\ell}X_{H-1}(S)_{*\ii}\in \RR^{d_H\times \bar{n}},\\
			X_{\ell}&=W_{\ell}X_{\ell-1}S\in \RR^{d_{\ell}\times n}, \qquad \text{for } 1\leq \ell\leq H-1.
		\end{align*}		
		
		The partial derivative can be calculated by the relation \eqref{eqn:vec}, and we get
		\begin{align*}
			\nabla_{\vc{W_1}} L(\W) &= \bigg( \frac{\partial L}{\partial \textup{vec}(\hat{Y})} \big[I_{\bar{n}}\otimes W_{H+1}\big]\big[(S)_{*\ii}^{\top}\otimes W_{H}\big]\big[S^{\top}\otimes W_{H-1}\big]\cdots \bigg. \\
			& \bigg. \hspace{6cm}  \big[S^{\top}\otimes W_2\big]\big[(XS)^{\top}\otimes I_{d_1}\big] \bigg)^{\top}\\
			&=  \frac{2}{m}\,\big[(XS^H)_{*\ii}\otimes W_2^{\top}W_{3}^{\top}\cdots W_{H+1}^{\top}\big]\text{vec}(\hat{Y}-Y)\\
			\nabla_{\vc{W_1}} L(\W)&= \frac{2}{m}\,\big[I_{d_x}\otimes W_2^{\top}W_{3}^{\top}\cdots W_{H+1}^{\top}\big]\vc{(\hat{Y}-Y)(XS^H)_{*\ii}^{\top}}.\numberthis\label{eqn:gradient_W1}
		\end{align*}
		The product $W_2^{\top}W_{3}^{\top}\cdots W_{H+1}^{\top}\in \RR^{d_1\times d_y}$ and $d_1\geq d_y$ implies
		\begin{align*}
			\|\nabla_{\vc{W_1}} L(\W)\|_2^2	&\geq \frac{4}{m^2}\, \sgmin^2\big( W_2^{\top}W_{3}^{\top}\cdots W_{H+1}^{\top} \big)\big\|\text{vec}\big( (\hat{Y}-Y)(XS^H)_{*\ii}^{\top} \big)\big\|_2^2\\
			&\geq \frac{4}{m^2}\, \sgmin^2\big( W_2^{\top}W_{3}^{\top}\cdots W_{H+1}^{\top} \big)\big\| \big[ (XS^H)_{*\ii}\otimes I_{d_y} \big]\text{vec}(\hat{Y}-Y) \big\|_2^2. \numberthis\label{eqn:lowerbound}
		\end{align*}
		
		Since $W_{\ell}(0)\in \RR^{d_{\ell}\times d_{\ell-1}}$ is a full rank with $d_{\ell}\leq d_{\ell-1}$ for $2\leq \ell\leq H+1$, then Lemma \ref{lemma:sgmin_sum} implies for each $W_{\ell}\in \RR^{d_{\ell}\times d_{\ell-1}}$ with $$\|W_{\ell}-W_{\ell}(0)\|_F\leq r_{\ell}:=\frac{\sgmin\big(W_{\ell}(0)\big)}{2}, \qquad 2\leq \ell\leq H+1,$$ $W_{\ell}$ is full rank with $\sgmin(W_{\ell})\geq \sgmin\big(W_{\ell}(0)\big)/2$.
		Moreover, Lemma \ref{lemma:sgmin_product} implies $\big(W_2^{\top}W_{3}^{\top}\cdots W_{H+1}^{\top}\big)$ is full rank with
		\begin{align}
			\label{eqn:sgmin_weights}
			\sgmin^2\big(W_2^{\top}W_{3}^{\top}\cdots W_{H+1}^{\top}\big)\geq \prod_{\ell=2}^{H+1} \sgmin^2(W_{\ell})\geq \frac{1}{4^H} \prod_{\ell=2}^{H+1} \sgmin^2\big(W_{\ell}(0)\big).
		\end{align}
		Now, the Lemma \ref{lemma:small_singular} implies
		\begin{align*}
			\big\| \big[ (XS^H)_{*\ii}\otimes I_{d_y} \big]\text{vec}(\hat{Y}-Y) \big\|_2^2 &\geq \sgs^2\big([ (XS^H)_{*\ii}\otimes I_{d_y} ]\big)\|\pp_{[(XS^H)_{*\ii}^{\top}\otimes I_{d_y}]}\text{vec}(\hat{Y}-Y)\|_2^2\\
			&=\sgs^2\big((XS^H)_{*\ii}\big)\big\|\text{vec}(\hat{Y})-\pp_{[(XS^H)_{*\ii}^{\top}\otimes I_{d_y}]}\text{vec}(Y)\big\|_2^2\\
			&\geq \sgs^2\big((XS^H)_{*\ii}\big)\Big(\|\text{vec}(\hat{Y}-Y)\|_2^2\\
			&\hspace{1.5in}-\big\|\text{vec}(Y)-\pp_{[(XS^H)_{*\ii}^{\top}\otimes I_{d_y}]}\text{vec}(Y)\big\|_2^2\Big)\\
			&\geq m\sgs^2\big((XS^H)_{*\ii}\big)\big(\ll(\bm{W})-\tilde{\ll}_H\big).
		\end{align*}
		Here $\pp_{[(XS^H)_{*\ii}^{\top}\otimes I_{d_y}]}$ denote the orthogonal projection onto column space of $\big[ (XS^H)_{*\ii}^{\top}\otimes I_{d_y} \big]$.

		Let $\vc{\W}\in \RR^p$ denotes the collection $\W$ in vector form and $B\big(\vc{\W(0)},r\big)$ be the closed ball of radius $r$ centered at $\vc{\W(0)}$ in $\RR^p$, where $p= \sum\limits_{\ell=1}^{H+1} d_{\ell}d_{\ell-1}$ and $r=\min\{r_{\ell}: 2\leq \ell\leq H+1\}$. For each $\vc{\W}\in B\big(\vc{\W(0)},r\big)$, $$\|W_{\ell}-W_{\ell}(0)\|_F\leq r, \qquad \text{for } 1\leq \ell\leq H+1.$$ 
		Hence, for any $\vc{\W}\in B\big(\vc{\W(0)},r\big)$, the inequalities \eqref{eqn:lowerbound} and \eqref{eqn:sgmin_weights} implies
		\begin{align}
			\label{eqn:gnn_alpha_bnd}
			\frac{\|\nabla_{\vc{W_1}} L(\W)\|_2^2}{L(\W)} &\geq \frac{1}{m}\beta\sgs^2\big((XS^H)_{*\ii}\big)
		\end{align}
		where $\beta= \frac{1}{4^{H-1}}\prod_{\ell=2}^{H+1} \sgmin^2(W_{\ell}(0))$. 
		
		Without loss of generality, assume that $\sgmin(W_{H+1}(0))$ is sufficiently large such that $$\frac{4}{r^2}L\big(\bm{W}(0)\big)\leq \frac{1}{m}\beta \sgs^2\big((XS^H)_{*\ii}\big).$$ This implies
		\begin{equation}\label{eqn:alphabdd}
			4L\big(\bm{W}(0)\big)\leq r^2\inf \Big\{ \frac{|\nabla L(\W)|^2}{L(\W)}:\vc{\W}\in B\big(\vc{\W(0)},r\big) \Big\}.
		\end{equation}
		Consequently, Theorem \ref{thm:sourav} signifies the gradient flow started at $\bm{W}(0)$ converges to some $\tilde{\bm{W}}$ such that $L(\tilde{\bm{W}})=0$. Moreover, at time $T>0$, 
		$$\sum_{\ell=1}^{H+1}\|W_{\ell}(T)-\tilde{W}_{\ell}\|_F^2\leq r^2\,\exp\Big(-\frac{1}{m}\beta\sgs^2((XS^H)_{*\ii})T\Big),$$ and $$\ll\big(\W(T)\big)-\tilde{\ll}_H\leq e^{-\frac{1}{m}\beta \sgs^2((XS^H)_{*\ii})T}\big( \ll(\W(0))-\tilde{\ll}_H \big).$$ This completes the proof.
	\end{proof}
	
	\begin{rem}\label{rem:initialisation}
		The convergence rate of the gradient flow training of linear GNNs explicitly depends on the singular value of weight parameters. Let $\W(0)$ be the collection of initial weight matrices such that
		\begin{align*}
			W_1(0) &: \text{ zero matrix},\\
			W_{\ell}(0) &: \text{ diagonal matrix with entries are 1, for } 2\leq \ell\leq H,\\
			W_{H+1}(0) &: \text{ diagonal matrix with entries are $a$},
		\end{align*}
		then $\beta=\frac{a^2}{4^{H-1}}$. Moreover, with the initialization, the gradient flow converges to the global minimum if $$a^2\geq \max\Big\{ 1,\frac{4^{H+1}m\big(\|Y\|_F^2-\tilde{\ll}_H\big)}{\sgs^2\big((XS^H)_{*\ii}\big)} \Big\}.$$ 
	\end{rem}
	
	
	\begin{rem}
		The convergence rate of the gradient flow explicitly depends on the feature matrix $X$, graph shift matrix $S$ and the initial weight matrices. The notation $\sgs^2((XS^H)_{*\ii})$ denotes the smallest non-zero singular value of $(XS^H)_{*\ii}$, which implies that the training loss for the semi-supervised learning of graph neural network converges to the global minimum with respect to aggregate matrix $S$ even if $(XS^H)_{*\ii}$ has the rank deficiency.
	\end{rem}
	
	\section{Energy optimization in linear GNNs}
	\label{sec:min_energy}
	In this section, we discuss the solution to the global minimum of the mean square loss of the linear GNN that minimizes the total energy of the weight matrices.
	
	Let $G=(V,E)$ be a graph embedded with a feature matrix $X\in \RR^{d_x\times n}$, and a labeled matrix $Y\in \RR^{d_y\times \bar{n}}$. The linear GNN with $H$ layers and weight matrices $\W$ is defined by \eqref{eqn:linear_GNN}. In general, the global minimum of the mean square loss $\ll(\W)$ of the linear GNN satisfy
	\begin{align*}
		\min_{\W}\ll(\W) \geq& \min_{W\in \RR^{d_y\times d_x}} \frac{1}{\bar{n}d_y}\big\|W(XS^H)_{*\ii}-Y\big\|_F^2.
	\end{align*}
	The equality holds if for each $1\leq \ell\leq H$, the hidden feature dimension $d_{\ell}\geq\min\{d_x,d_y\}$.
	
	Let $W\in \RR^{d_y\times d_x}$ be arbitrary, and let the rank of $(XS^H)_{*\ii}\in \RR^{d_x\times \bar{n}}$ be $k$. Then, from the singular value decomposition of $(XS^H)_{*\ii}$ we have
	\begin{align*}
		\big\|W(XS^H)_{*\ii}-Y\big\|_F^2=&\big\|WU\Sigma V^{\top}-Y\big\|_F^2 = \|WU\Sigma-YV\|_F^2\\
		=& \sum_{i=1}^{d_y}\sum_{j=1}^{k} \Big( \sigma_j\big((XS^H)_{*\ii}\big)[WU]_{ij} - [YV]_{ij} \Big)^2+ \sum_{i=1}^{d_y}\sum_{j=k+1}^{\bar{n}} [YV]_{ij}^2,
	\end{align*}
	where $[P]_{ij}$ denotes the $i,j$-th entry of the matrix $P$. Hence, the global minimum of least square equation $\|W(XS^H)_{*\ii}-Y\|_F^2$ is attained for all $W\in \RR^{d_y\times d_x}$ such that $$\big[ WU \big]_{ij}=\frac{1}{\sigma_j((XS^H)_{*\ii})}[YV]_{ij}, \qquad \text{for all } 1\leq i\leq d_y,\, 1\leq j\leq k.$$ In particular, the global minimum solution $W$ also minimizes the total energy $\|W\|_F$ if and only if $[WU]_{ij}=0$ for $1\leq i\leq d_y$, $k+1\leq j\leq \bar{n}$. Hence, $\tilde{W}$ is the unique solution to the global minimum of $\ll$ with minimum energy and $\tilde{W}$ is given by $$\tilde{W}=Y(XS^H)_{*\ii}^{\dagger}.$$
	
	The problem of identifying the solution to the minimization of square loss with minimum energy of the weights for general linear GNN is non-trivial. Let the collection $\tilde{\W}=\big( \tilde{W}_1,\tilde{W}_2,\dots,\tilde{W}_{H+1} \big)$ be the solution to the global minimum of the square loss. Then $\tilde{\W}$ minimizes the energy of the product $\tilde{W}_{H+1}\tilde{W}_{H}\cdots\tilde{W}_{1}$ if and only if
	\begin{align}
		\label{eqn:min_weight_cond}
		\tilde{W}_{H+1}\tilde{W}_{H}\cdots\tilde{W}_{1}&=Y(XS^H)_{*\ii}^{\dagger}. 
	\end{align}
	However, for $H\geq 1$, the collection $\tilde{\W}$ might not minimize the energy of the individual weights matrices.
	
	In the following theorem, we address the optimization problem for the linear GNNs with $H$ layers defined as follows.
	\begin{equation}
		\label{eqn:energy_opti}
		\begin{aligned}
			\min_{\tilde{W}} \sum_{\ell=1}^{H+1} &\|\tilde{W}_{\ell}\|_F^2\\
			\hspace{-1in}\text{subject to } \tilde{W}_{H+1}\tilde{W}_{H}\cdots\tilde{W}_{1}&=Y(XS^H)_{*\ii}^{\dagger}.
		\end{aligned}
	\end{equation}
	
	Before answering the above optimization problem, first we define the \textit{balanced} condition on the collections of weight matrices.
	\begin{defin}
		The collection $\W=\big( W_1, W_2,\dots, W_H \big)$ is said to be \textit{balanced} if for every $1\leq \ell\leq H$, $$W_{\ell}W_{\ell}^{\top}=W_{\ell+1}^{\top}W_{\ell+1}.$$
	\end{defin}
	
	\begin{thm}
		Let $G=(V,E)$ be a graph embedded with a feature matrix $X\in \RR^{d_x\times n}$, and a labeled matrix $Y\in \RR^{d_y\times \bar{n}}$. Let us consider a linear GNN with $H$ layers and $S$ be a graph shift matrix. If the collection $\W$ is balanced and satisfy the equation \eqref{eqn:min_weight_cond}, then $\W$ is the solution to the optimization problem \eqref{eqn:energy_opti}. 
	\end{thm}
	
	\begin{proof}
		Let $\W$ be the collection of weight matrices of the linear GNN with $H$ layers. In order to solve the optimization problem, we use the method of Lagrange multiplier and the Lagrangian function is defined by $$\mathcal{F}(\W) = \sum_{\ell=1}^{H+1} \|W_{\ell}\|_F^2 + \text{trace}\big(\Lambda^{\top}(W_{H+1}W_{H}\cdots W_{1}-Y(XS^H)_{*\ii}^{\dagger}) \big),$$ where $\Lambda\in \RR^{d_y\times d_x}$ is the Lagrange multiplier. All the critical points of $\mathcal{F}$ is given by
		\begin{equation}
			\label{eqn:critical_points}
			\frac{\partial \mathcal{F}}{\partial \Lambda}=0, \qquad \text{and} \qquad \frac{\partial \mathcal{F}}{\partial W_{\ell}}=0, \quad \text{for }1\leq \ell\leq H+1.
		\end{equation}
		The fact $\frac{\partial}{\partial Q}\text{trace}(PQR)= P^{\top}R^{\top}$ for any $P\in \RR^{d\times d'}$, $Q\in \RR^{d'\times d''}$ and $R\in \RR^{d''\times d'''}$, and the equation \eqref{eqn:critical_points} implies that the critical points of $\mathcal{F}$ is balanced and satisfy \eqref{eqn:min_weight_cond}.\par 
		Let $\W$ be a critical points of $\mathcal{F}$. From the balanced condition of $\W$ and \eqref{eqn:min_weight_cond}, one can easily verify that
		\begin{align*}
			\begin{aligned}
				\big(Y(XS^H)_{*\ii}^{\dagger}\big)^{\top}\big(Y(XS^H)_{*\ii}^{\dagger}\big)&= \big(W_1^{\top}W_1\big)^{H+1},\\
				\big(Y(XS^H)_{*\ii}^{\dagger}\big)\big(Y(XS^H)_{*\ii}^{\dagger}\big)^{\top}&= \big(W_{H+1}W_{H+1}^{\top}\big)^{H+1}.
			\end{aligned}		
		\end{align*}
		Hence, the eigenspace of $W_1^{\top}W_1$ and $\big(Y(XS^H)_{*\ii}^{\dagger}\big)^{\top}\big(Y(XS^H)_{*\ii}^{\dagger}\big)$ are identical. In particular, if $\lambda$ is the eigenvalue of $\big(Y(XS^H)_{*\ii}^{\dagger}\big)^{\top}\big(Y(XS^H)_{*\ii}^{\dagger}\big)$ then $\lambda^{\frac{1}{H+1}}$ is the eigenvalue of $W_1^{\top}W_1$. Hence, the minimum value of the total energy of weight matrices for the optimization problem \eqref{eqn:energy_opti} is given by $$\sum_{\ell=1}^{H+1} \|W_{\ell}\|_F^2=(H+1)\|W_1\|_F^2=(H+1)\sum_{i=1}^{\min\{d_x,d_y\}} \sigma_i^{\frac{2}{H+1}}\big(Y(XS^H)_{*\ii}^{\dagger}\big).$$
		
		This completes the proof.
	\end{proof}
	
	\begin{rem}
		The balancedness condition of the collection $\W$ provide algebraic relation between weight matrices. For example, suppose $\sigma$ is a non-zero singular value of $W_{1}$, then the relation $W_1W_1^{\top}=W_2^{\top}W_2$ implies $\sigma$ is also a singular value of $W_2$. In particular, $\sigma$ is the non-zero singular value of $W_{\ell}$ for each $1\leq \ell\leq H+1$. Hence, all weight matrices $W_{\ell}$ share the same set of non-zero singular values. 
	\end{rem}
	
	
	The global minimum of \eqref{eqn:trainingloss} with balanced condition, minimizes the total energy of the weights. This initiates another non-trivial challenge, whether the initialization leads the gradient flow to the solution of \eqref{eqn:energy_opti}.
	The articles \cite{arora2018convergence,bah2022learning,yacine2023geometric} studied closely into this direction for linear deep neural networks with approximately balanced initialization. However, they did not discuss on the optimization problem \eqref{eqn:energy_opti} and studied the convergence of the gradient dynamics with full rank assumption on the input data.
	Even though the convergence analysis discussed in Section \ref{sec:grad_flow} provides a sufficient initialization for the convergent of the square loss to the global minimum, but it might not optimize the total energy of the weight matrices.	In the following theorem, we answered this question with some weaker assumption on the hidden dimension.
	
	\begin{thm}
		Let $G=(V,E)$ be a graph embedded with a feature matrix $X\in \RR^{d_x\times n}$, and a labeled matrix $Y\in \RR^{d_y\times \bar{n}}$ with $\text{rank}\,\big(Y(XS^H)_{*\ii}^{\top}\big)=\b{k}$. Let us consider a linear GNN with $H$ layers such that $d_y\leq \min\{d_1,d_2,\dots,d_{H}\}$ and $S$ as aggregation matrix. The gradient flow started at $\W(0)$ converges to the solution of the optimization problem \eqref{eqn:energy_opti} if the following holds.
		\begin{enumerate}
			\item $\W(0)$ is balanced.
			\item If $\text{rank}\,\big((XS^H)_{*\ii}\big)=k$ and the orthonormal matrix $U\in \RR^{d_x\times d_x}$ is the collection of left singular vector of $(XS^H)_{*\ii}$, then for all $1\leq i\leq d_1,\, k<j\leq d_x$, $\big[ W_1(0)U \big]_{ij}=0$,
			\item Rank of $W_1(0)$ is $\b{k}$ with $\b{k}$-th singular value satisfy
			\begin{equation}
				\label{eqn:bal_init_cond}
				L\big(\W(0)\big)\leq \frac{1}{4^{H+1}m}\sigma_{\b{k}}^{2H}(W_{1}(0))\sgs^2\big((XS^H)_{*\ii}\big)
			\end{equation}
			where $\b{k}=\min\{d_y, k\}$, $m=\bar{n}d_x$, and $L$ is the loss function defined by \eqref{eqn:L_loss}.
		\end{enumerate}
	\end{thm}
	
	\begin{proof}
		In linear deep neural network with squared loss, the balanced relation is preserved by the gradient flow equation, see \cite{arora2018convergence,bah2022learning,yacine2023geometric}. In the case of linear graph neural networks, it is easy to check that for $1\leq \ell\leq H+1$ $$\frac{dW_{\ell}}{dt}=-\frac{2}{\bar{n}d_y}\big(W_{H+1}W_{H}\cdots W_{\ell+1}\big)^{\top}(\hat{Y}-Y)(XS^H)_{*\ii}^{\top}\big(W_{\ell-1}W_{\ell-2}\cdots W_1\big)^{\top}.$$ Hence, we have
		\begin{equation*}
			\frac{d}{dt} \big( W_{\ell}W_{\ell}^{\top}\big)=\frac{d}{dt} \big(W_{\ell+1}^{\top}W_{\ell+1}\big), \qquad \forall 1\leq \ell\leq H.
		\end{equation*}
		If the weight parameters are initially balanced then $\W(t)$ is balanced for each $t\geq 0$.
		
		Let us assume that the gradient flow converges to the global minimum at $\tilde{\W}$. The global minimizer $\tilde{\W}$ satisfy \eqref{eqn:min_weight_cond} if and only if $$\big[\tilde{W}_{H+1}\tilde{W}_H\cdots \tilde{W}_1U\big]_{ij}=0 \qquad \text{for all } 1\leq i\leq d_y,\, k< j\leq d_x.$$
		Hence it is enough to prove that at any time $t\geq 0$, $\W(t)$ satisfy $$\big[W_{H+1}W_H\cdots W_1U\big]_{ij}(t)=0 \qquad \text{for all } 1\leq i\leq d_y,\, k< j\leq d_x.$$ First, we observe that for $\ell=1$ the gradient flow equation implies
		\begin{align*}
			\frac{d(W_1U)}{dt}&= -\nabla_{W_1} \ll(\W)U\\
			&= -\frac{2}{m}\big( W_{H+1}W_{H}\cdots W_{2} \big)^{\top}\big(W_{H+1}W_{H}\cdots W_{1}(XS^H)_{*\ii}-Y\big)V\Sigma^{\top}.
		\end{align*}
		This implies, $\Big[ \frac{d(W_1U)}{dt} \Big]_{ij}=0$ for $1\leq i\leq d_1$ and $k<j\leq d_x$, and therefore $$\big[W_1U\big]_{ij}(t)=\big[W_1U\big]_{ij}(0)=0, \qquad \forall\, t\geq 0.$$
		
		Suppose there exists $\ell\geq 1$ such that for $1\leq i\leq d_\ell$, and $k<j\leq d_x$, we have $$\big[W_{\ell}W_{\ell-1}\cdots W_1U\big]_{ij}(t)=0, \qquad \text{for } t\geq 0.$$ The chain rule for differentiation implies
		\begin{align*}
			\frac{d}{dt} \big( W_{\ell+1}W_{\ell}\cdots W_1U \big)&=\frac{dW_{\ell+1}}{dt}W_{\ell}\cdots W_1U+ W_{\ell+1}\frac{d}{dt}\big( W_{\ell}\cdots W_1U \big).
		\end{align*}
		Hence, for all $t\geq 0$, $1\leq i\leq d_{\ell+1}$ and $k<j\leq d_x$
		\begin{align*}
			\Big[ \frac{d}{dt} \big( W_{\ell+1}W_{\ell}\cdots W_1U \big) \Big]_{ij}&=0,\\
			\big[ W_{\ell+1}W_{\ell}\cdots W_1U \big]_{ij}(t)&= 0.
		\end{align*}
		Therefore, from mathematical induction we have 
		\begin{align*}
			\big[ W_{H+1}W_{H}\cdots W_1U \big]_{ij}(t)&= 0, \qquad \forall\, t\geq 0,\, 1\leq i\leq d_y,\, k< j\leq d_x.
		\end{align*}
		
		Finally, we now ensure convergence of the gradient flow training of linear GNN with balanced initialization and some minor assumption on $W_1(0)$. In order to study the convergence result, we consider the following two cases.\par 
		First assume that $d_y\leq \text{rank}\,\big((XS^H)_{*\ii}\big)$. The equation \eqref{eqn:gradient_W1} implies 
		\begin{align*}
			\|\nabla_{\text{vec}\,(W_1)} L(\W)\|_2^2	&\geq \frac{4}{m^2}\sigma_{d_y}^2\big(W_2^{\top}W_3^{\top}\cdots W_{H+1}^{\top}\big)\big\| (\hat{Y}-Y)(XS^H)_{*\ii}^{\top} \big\|_F^2\\
			&= \frac{4}{m^2}\sigma_{d_y}^{2H}\big(W_{H+1}\big)\big\| (\hat{Y}-Y)(XS^H)_{*\ii}^{\top} \big\|_F^2.
		\end{align*}
		The last equality follows from the balancedness of $\W(t)$ for each $t\geq 0$.\par 
		Secondly, assume that $k:=\text{rank}\,\big((XS^H)_{*\ii}\big)\leq d_y$ and $(XS^H)_{*\ii}=U\Sigma V^{\top}$ is the singular value decomposition of $(XS^H)_{*\ii}$. Then the gradient of $L$ with respect to $\vc{W_{H+1}}$ implies
		\begin{align*}
			\|\nabla_{\text{vec}\,(W_{H+1})} L(\W)\|_2^2 &= \frac{4}{m^2}\big\| (\hat{Y}-Y)V\Sigma^{\top}U^{\top}W_1^{\top}W_2^{\top}\cdots W_{H}^{\top} \big\|_F^2\\
			&= \frac{4}{m^2}\big\| W_{H}W_{H-1}\cdots W_1 U \Sigma V^{\top}(\hat{Y}-Y)^{\top} \big\|_F^2.
		\end{align*}
		The entries of $j$-th column with $j>k$ in the matrix $W_{H}W_{H-1}\cdots W_1 U\in \RR^{d_H\times d_x}$ are zero. Simultaneously, the entries of $i$-th row with $i>k$ in the matrix $\Sigma V^{\top}(\hat{Y}-Y)^{\top}$ are zero. Hence, we get
		\begin{align*}
			\|\nabla_{\text{vec}\,(W_{H+1})} L(\W)\|_2^2 &= \frac{4}{m^2}\big\|W\bar{Y}\big\|_F^2\geq \frac{4}{m^2}\sgmin^2(W)\|\bar{Y}\|_F^2\\
			&= \frac{4}{m^2} \sigma_{k}^2(W_{H}W_{H-1}\cdots W_1 U)\|(\hat{Y}-Y)V\Sigma^{\top}\|_F^2\\
			&= \frac{4}{m^2}\sigma_{k}^{2H}(W_1)\big\| (\hat{Y}-Y)(XS^H)_{*\ii}^{\top} \big\|_F^2.
		\end{align*}
		In the first equality $W\in \RR^{d_H\times k}$ and $\bar{Y}\in \RR^{k\times d_y}$ are the matrix with first $k$ columns of $W_{H}W_{H-1}\cdots W_1 U$ and first $k$ rows of $\Sigma V^{\top}(\hat{Y}-Y)^{\top}$ respectively.\par 
		
		If the $\W(0)$ is balanced with $\text{rank}\,(W_{\ell}(0))=\b{k}$ and $j$-th column of $W_1(0)U$ is a zero vector for each $j>\text{rank}\,\big((XS^H)_{*\ii}\big)$, then each $\W_{\ell}(t)$ in the collection $\W(t)$ shares the set of non-zero singular values for $t\geq 0$. Therefore, following the similar arguments as in Theorem \ref{thm:gradient_flow}, the gradient flow converges to the global minimum provided the $\b{k}$-th singular value of $W_{\ell}(0)$ satisfy $$L(\W(0))\leq \frac{1}{4^{H+1}m}\sigma_{\b{k}}^{2H}(W_{\ell}(0))\sgs^2\big((XS^H)_{*\ii}\big).$$
		This completes the proof.
	\end{proof}
	
	The assumption \eqref{eqn:bal_init_cond} on the initial weights is equivalent to the inequality \eqref{eqn:alphabdd} for balanced initialization. The $\b{k}$-th singular value of the weights at time $t\geq 0$ in the gradient flow training always greater than $\sigma_{\b{k}}(W_{\ell}(0))/2$ and the weights converges to the solution of optimization problem \eqref{eqn:energy_opti}. Therefore, the balanced initialization satisfy $$\sigma_{\b{k}}(W_{\ell}(0))\leq 2\,\sigma_{\b{k}}^{\frac{1}{H+1}}\big(Y(XS^{H})_{*\ii}\big).$$ The balanced condition imposes such restriction on the initialization which causes the inequality \eqref{eqn:bal_init_cond} infeasible.
	
	\section{Convergence analysis of gradient descent training}\label{sec:grad_desc}
	In this section, we study the gradient descent training of linear GNNs and discuss its convergence. The gradient descent training is an analogue of gradient flow training of GNNs. The weight parameters are updated in the direction of the negative gradient with infinitesimally small steps in a gradient flow training, however, stepsize is fixed in the gradient descent training. Therefore, with an inappropriate stepsize the weight parameter diverges to infinity in gradient descent method.
	
	The following theorem is the analogue of Theorem \ref{thm:gradient_flow} for gradient descent. The idea of the proof is to find an appropriate stepsize such that the weights in each iterations are lies in the neighborhood of the initial weights. The loss $\ll$ corresponds to the weight sequences is a decreasing sequence bounded below by global minimum.
	\begin{thm}\label{thm:grad_desc}
		Let $G=(V,E)$ be a graph embedded with a feature matrix $X\in \RR^{d_x\times n}$, and a labeled matrix $Y\in \RR^{d_y\times \bar{n}}$. Let us consider a linear GNN with $H$ layers and $S$ be the graph shift matrix, and consider the mean square loss $\ll$ define by \eqref{eqn:trainingloss}. With a sufficiently small stepsize $\eta>0$ and the initial weights considered in Theorem \ref{thm:gradient_flow}, the gradient descent 
		\begin{align}
			\label{eqn:iteration_W}
			W_{\ell}^{(k+1)}=W_{\ell}^{(k)}-\eta\nabla_{W_{\ell}} \ll\big(\W^{(k)}\big), \qquad \text{for } 1\leq \ell\leq H+1,
		\end{align}
		converges to global minimum of $\ll$ as $k\to \infty$. Here $\W^{(0)}=\W(0)$.
		
		Moreover, for $k\geq 0$, 
		$$\ll(\W^{(k)})-\tilde{\ll}_H\leq \bigg( 1-\frac{\beta\sgs^2\big((XS^H)_{*\ii}\big)\eta}{2m} \bigg)^k \big( \ll(\W^{(0)})-\tilde{\ll}_H \big).$$
	\end{thm}
	
	\begin{proof}
		For $1\leq \ell\leq H+1$, the gradient descent algorithm \eqref{eqn:iteration_W} can be written as follows:
		\begin{equation}
			\label{eqn:iteration_wvec}
			\vc{\W^{(k+1)}}=\vc{\W^{(k)}}-\eta\nabla \ll\big(\vc{\W^{(k)}}\big),
		\end{equation}
		where $\vc{\W^{(k)}}$ is a vector form of the collection $\big(W_1^{(k)},W_2^{(k)},\dots,W_{H+1}^{(k)}\big)$.\par
%
		The loss function $\ll$ is defined by \eqref{eqn:trainingloss}, is twice continuously differentiable. Hence, there exists a positive constant $M$ such that for each $\textup{vec}(\W)$ in a closed bounded ball $B_{2r}\big(\vc{\W^{(0)}}\big)\subset \RR^p$, absolute values of the entries in the gradient and the Hessian matrix of $\ll$ is bounded by $M$, i.e.,
		\begin{gather*}
			\Big| \frac{\partial \ll(\vc{\W})}{\partial \vc{\W}_j} \Big|\leq M \qquad \text{and} \qquad \Big|\frac{\partial^2 \ll(\vc{\W})}{\partial \vc{\W}_i \partial \vc{\W}_j}\Big|\leq M
		\end{gather*}
		where $r$ and $p$ are the same as define in Theorem \ref{thm:gradient_flow}.
		
		
		Let $k\geq 1$ be such that $\vc{\W^{(j)}}\in B_r\big(\vc{\W^{(0)}}\big)$ for each $1\leq j\leq k$. Claim that for sufficiently small $\eta$, $\vc{\W^{(k+1)}}\in B_r\big(\vc{\W^{(0)}}\big)$.\par
		For $\eta<\frac{r}{M\sqrt{p}}$ and \eqref{eqn:iteration_wvec} implies
		\begin{align*}
			\big\|\vc{\W^{(j+1)}}-\vc{\W^{(j)}}\big\|_2&= \eta\big\|\nabla \ll\big(\vc{\W^{(j)}}\big)\big\|_2\leq \eta M\sqrt{p}<r.
		\end{align*}
		Hence, $\vc{\W^{(j+1)}}\in B_{2r}\big(\vc{\W^{(0)}}\big)$, and the line segment joining $W_{\ell}^{(j+1)}$ and $W_{\ell}^{(j)}$ is in $B_{2r}\big(\vc{\W^{(0)}}\big)$.
		
		Let $R_j:=\ll\big(\vc{\W^{(j+1)}}\big)-\ll\big(\vc{\W^{(j)}}\big)+\eta\|\nabla \ll\big(\vc{\W^{(j)}}\big)\|_2^2$. 
		The second order expansion of $\ll$ implies
		\begin{align*}
			R_j
			&= \frac{1}{2}\eta^2\nabla \ll\big(\vc{\W^{(j)}}\big)\cdot \nabla^2 \ll\big(\vc{\W^*}\big)\nabla \ll\big(\vc{\W^{(j)}}\big)\\
			&\leq \frac{1}{2}\eta^2 Mp \|\nabla \ll(\vc{\W^{(j)}})\|_2^2\\
			&\leq \frac{1}{2}\eta \|\nabla \ll(\vc{\W^{(j)}})\|_2^2
		\end{align*} 
		where $\vc{\W^*}$ is some point in the line segment joining $\vc{\W^{(j+1)}}$ and $\vc{\W^{(j)}}$, and $\eta<\min\{\frac{1}{Mp},\frac{r}{M\sqrt{p}}\}$.
		
		Hence, the training loss $\ll$ decreases at each iterations as
		\begin{equation}
			\label{eqn:gradient_bdd_disc}
			\frac{1}{2}\eta \big\|\nabla \ll\big(\vc{\W^{(j)}}\big) \big\|_2^2 \leq \ll\big(\vc{\W^{(j)}}\big)-\ll\big(\vc{\W^{(j+1)}}\big).
		\end{equation}
		Moreover, for $1\leq j\leq k$, we have
		\begin{align*}
			\ll(\vc{\W^{(j+1)}}) &= \ll(\vc{\W^{(j)}})-\eta\|\nabla \ll(\vc{\W^{(j)}})\|_2^2+R_j\\
			&\leq \ll(\vc{\W^{(j)}})-\frac{1}{2}\eta \|\nabla \ll(\vc{\W^{(j)}})\|_2^2\\
			\Big(\ll\big(\vc{\W^{(j+1)}}\big)-\tilde{\ll}_H\Big) &\leq \Big(1-\frac{\eta\alpha}{2}\Big)\Big(\ll\big(\vc{\W^{(j)}}\big)-\tilde{\ll}_H\Big)\\
			\Big(\ll\big(\vc{\W^{(j+1)}}\big)-\tilde{\ll}_H\Big) &\leq \Big(1-\frac{\eta\alpha}{2}\Big)^{j+1}\Big(\ll\big(\vc{\W^{(0)}}\big)-\tilde{\ll}_H\Big) \numberthis \label{eqn:loss_rate_j}
		\end{align*}
		where $\alpha:=\inf\left\{ \frac{|\nabla \ll(\W)|^2}{\ll(\W)-\tilde{\ll}_H}:\vc{\W}\in B_r\big(\vc{\W^{(0)}}\big) \right\}$.
		
		For $1\leq j\leq k$, and the gradient descent algorithm \eqref{eqn:iteration_wvec} implies
		\begin{align}
			\label{eqn:weight_difference}
			\big\|\vc{\W^{(k+1)}}-\vc{\W^{(j)}}\big\|_2&\leq \sum_{i=j}^{k} \eta\big\|\nabla \ll\big(\vc{\W^{(i)}}\big)\big\|_2.
		\end{align}
		Let $L$ be the relative loss of $\ll$ with respect to global minimum, i.e., $$L\big(\vc{\W}\big):=\ll(\vc{W})-\tilde{\ll}_H.$$ The upper bound of the sum of the gradient is estimated by \eqref{eqn:gradient_bdd_disc} and \eqref{eqn:loss_rate_j}.
		\begin{align*}
			\sum_{i=j}^{k} \eta\big\|\nabla \ll\big(\vc{\W^{(i)}}\big)\big\|_2&\leq \sum_{i=j}^{k} \sqrt{2\eta \Big( L\big(\vc{\W^{(i)}}\big)-L\big(\vc{\W^{(i+1)}}\big) \Big)}\\
			&= \sqrt{2\eta}\sum_{i=j}^{k} \Big( \sqrt{L\big(\vc{\W^{(i)}}\big)}+\sqrt{L\big(\vc{\W^{(i+1)}}\big)} \Big)^{\frac{1}{2}}\times\\
			&\hspace{1.5in}\Big( \sqrt{L\big(\vc{\W^{(i)}}\big)}-\sqrt{L\big(\vc{\W^{(i+1)}}\big)} \Big)^{\frac{1}{2}}\\
			&\leq \sqrt{4\eta}\Big(L\big(\vc{\W^{(j)}}\big)\Big)^{\frac{1}{4}}\Big( \sum_{i=j}^{k} \sqrt{L\big(\vc{\W^{(i)}}\big)} \Big)^{\frac{1}{2}}\\
			&\leq \Big(4\eta L\big(\vc{\W^{(0)}}\big)\Big)^{\frac{1}{2}}\bigg( \big(1-\frac{\eta\alpha}{2}\big)^{\frac{j}{2}}\sum_{i=j}^{k} \big(1-\frac{\eta\alpha}{2}\big)^{\frac{i}{2}} \bigg)^{\frac{1}{2}}\\
			&\leq \Big(4\eta L\big(\W^{(0)}\big)\Big)^{\frac{1}{2}}\Big( \big(1-\frac{\eta\alpha}{2}\big)^{j}\sum_{i=0}^{k-j} \big(1-\frac{\eta\alpha}{2}\big)^{\frac{i}{2}} \Big)^{\frac{1}{2}}\\
			&\leq \big(1-\frac{\eta\alpha}{2}\big)^{\frac{j}{2}}\Big(4\eta L\big(\W^{(0)}\big)\Big)^{\frac{1}{2}}\Big( \sum_{i=0}^{k-j} \big(1-\frac{\eta\alpha}{4}\big)^{i} \Big)^{\frac{1}{2}}\\
			&\leq \Big(1-\frac{\eta\alpha}{2}\Big)^{\frac{j}{2}}\Big( \frac{16L(\W^{(0)})}{\alpha} \Big)^{\frac{1}{2}}
		\end{align*}
		The inequality \eqref{eqn:gnn_alpha_bnd} implies $\alpha$ is bounded below by $\frac{1}{m}\beta\sgs\big((XS^H)_{*\ii}\big)$.\\
		Let $\eta<\min\Big\{ \frac{r}{M\sqrt{p(H+1)}},\frac{1}{Mp},\frac{2m}{\beta \sgs^2((XS^H)_{*\ii})} \Big\}$, then the inequality \eqref{eqn:weight_difference} implies
		$$\big\|\vc{\W^{(k+1)}}-\vc{\W^{(j)}}\big\|_2 \leq \Big(1-\frac{\beta\sgs\big((XS^H)_{*\ii}\big)\eta}{2m}\Big)^{\frac{j}{2}}\Big(\frac{16mL\big(\W^{(0)}\big)}{\beta\sgs\big((XS^H)_{*\ii}\big)}\Big)^{\frac{1}{2}}.$$
		Moreover, if the initial choice of $\W(0)$ with $\sgmin(W_{H+1}(0))$ is sufficiently large such that $$\frac{16L(\W^{(0)})}{r^2}< \frac{1}{m}\beta \sgs^2\big((XS^H)_{*\ii}\big),$$
		then $\vc{\W^{(k)}}\in B_r\big(\vc{\W^{(0)}}\big)$. Hence, by mathematical induction $$\vc{\W^{(k)}}\in B_r\big(\vc{\W^{(0)}}\big) \qquad \forall\, k\geq 0.$$
		
		The sequence of weights $\{\vc{\W^{(k)}}\}_{k\geq 0}$ is a bounded contraction sequence, hence, a Cauchy sequence. In particular, for $\varepsilon>0$ there exists $M>0$ such that $$\big\|\vc{\W^{(k)}}-\vc{\W^{(j)}}\big\|_2<\varepsilon, \qquad k,j> M,$$ where $M=2\log\Big(\frac{r}{\varepsilon\sqrt{H+1}}\Big)/\log\Big(\frac{2m}{2m-\beta \sgs^2((XS^H)_{*\ii})\eta}\Big)$.
		This implies, weights converges to some $\vc{\tilde{\W}}\in B_r\big(\vc{\W^{(0)}}\big)$.
		
		Moreover, the inequality \eqref{eqn:loss_rate_j} implies
		$$\ll(\W^{(k)})-\tilde{\ll}_H\leq \Big( 1-\frac{\beta \sgs^2((XS^H)_{*\ii})\eta}{2m} \Big)^k \big( \ll(\W^{(0)})-\tilde{\ll}_H \big).$$
		
		This completes the proof.
	\end{proof}
	
	
	\begin{rem}
		Let $\varepsilon$ be any arbitrary small positive real. Theorem \ref{thm:grad_desc} shows that loss $\ll$ converges to the global minimum $\tilde{\ll}_H$ exponentially as the iteration progresses. In particular, $\ll$ is close to $\tilde{\ll}_H$ with $\varepsilon$ perturbation, i.e., $\ll(\W^{(k)})-\tilde{\ll}_H< \varepsilon$, if $$k>\bigg(\log \frac{\ll(\W^{(0)})-\tilde{\ll}_H}{\varepsilon}\bigg)/\bigg(\log \frac{2m}{2m-\beta \sgs^2((XS^H)_{*\ii})\eta}\bigg).$$
	\end{rem}
	

\section{Numerical experiments}
\label{sec:numerical}

	In this section, we illustrate aspects of the gradient flow dynamics such as the dependence of the convergence rate on the graph topology and the choice of the shift operator \( S \)  in terms of the principle singular value \( \sgs \left( ( X S^H)_{*\ii} \right)\). We provide a set of numerical experiments on the synthetic data following well-established undirected graph models (Erd\H{o}s--R\'enyi, \(k\)-nearest neighbors, SBM, Barab\'asi-Albert) and the real-world dataset of CDC climate data in US counties \cite{CDCdataset}.

\subsection{ Synthetic data }

Let us briefly describe synthetic graph models we use to illustrate the convergence of the linear GNN (see corresponding examples in Figure~\ref{fig:models}):
\begin{enumerate}[label=(\roman*)]
	\item \textbf{Erd\H{o}s--R\'enyi graph, \( G( n, p ) \)}: A random graph on \( n \) nodes where the edge \( \{ i, j\} \) enters the graph with probability \( p \) independently on all the other edges. The probability \( p \) naturally defines the connectivity of the generated graph; we assume \( p > (1 + \varepsilon ) \ln n / n \) such that the generated graph is almost surely connected, \cite{erdos1960evolution}.
	
	\item \textbf{\(k\)-nearest neighbors graph, \( K_k ( n ) \)}: A regular, highly-symmetric deterministic graph on \( n \) nodes. Assume all \( n \) nodes to be uniformly and equidistantly placed onto a unit circle; then each nodes is set to be connected by an edge to its \( k\) nearest neighbors (in Eucledean distance). For simplicity, we assume \( k \) to be even.
	
	\item \textbf{Stochastic Block Model (SBM)}: A structured generalization of \( G(n, p)\) model where \( n \) nodes are divided into \( r \) classes (blocks) with the probability matrix \( P \in \RR^{ r \times r } \) containing inter- and intra-classes edge probabilities, \cite{holland1983stochastic}. In the scope of the current work we assume two block setup \( \{ n_1, n_2 \} \) (such that \( n = n_1 + n_2 \)) with \( P = {\begin{pmatrix} p & q \\ q & p \end{pmatrix} } \) and \( p \gg q \); we refer to such models as \( \textrm{SBM}( n_1, n_2, p, q)\).
	
	\item \textbf{Barab\'asi-Albert model,} \( \text{BA}(n, m)\): A scale-free random graph model on \( n \) nodes where the node degree distribution follows the power law \( k^\gamma \) with \( 2 < \gamma < 3 \), \cite{barabasi1999emergence}. Parameter \( m \) correspond to the degree of the nodes added to the core graph during generation and typically regulates the minimal node degree.
\end{enumerate}

The feature data \( X \in \RR^{d_x \times n} \) is assumed to standard normal i.i.d., \( x_{ij} \sim \nn ( 0, 1 ) \), for the sake of simplicity. The output label data \( Y \) is set to be one-dimensional, \( d_y = 1 \), with \( y_i = f( x_i ) + \epsilon \sum_{j: \{i,j \} \in E } g(x_j) \) where \( f, g : \RR^{d_x} \mapsto \RR^{d_y} \) and the value of \( \epsilon \) controls the impact of the graph structure on the output label. We set \( f(x_i) = g (x_i) = \sum x_{ij} \) and \( \epsilon = 0.1 \) for experiments below and assume that the gradient flow is trained with respect to the mean square error (MSE).

\begin{figure}[hbtp]
\centering
\includegraphics[width = 1.0\columnwidth]{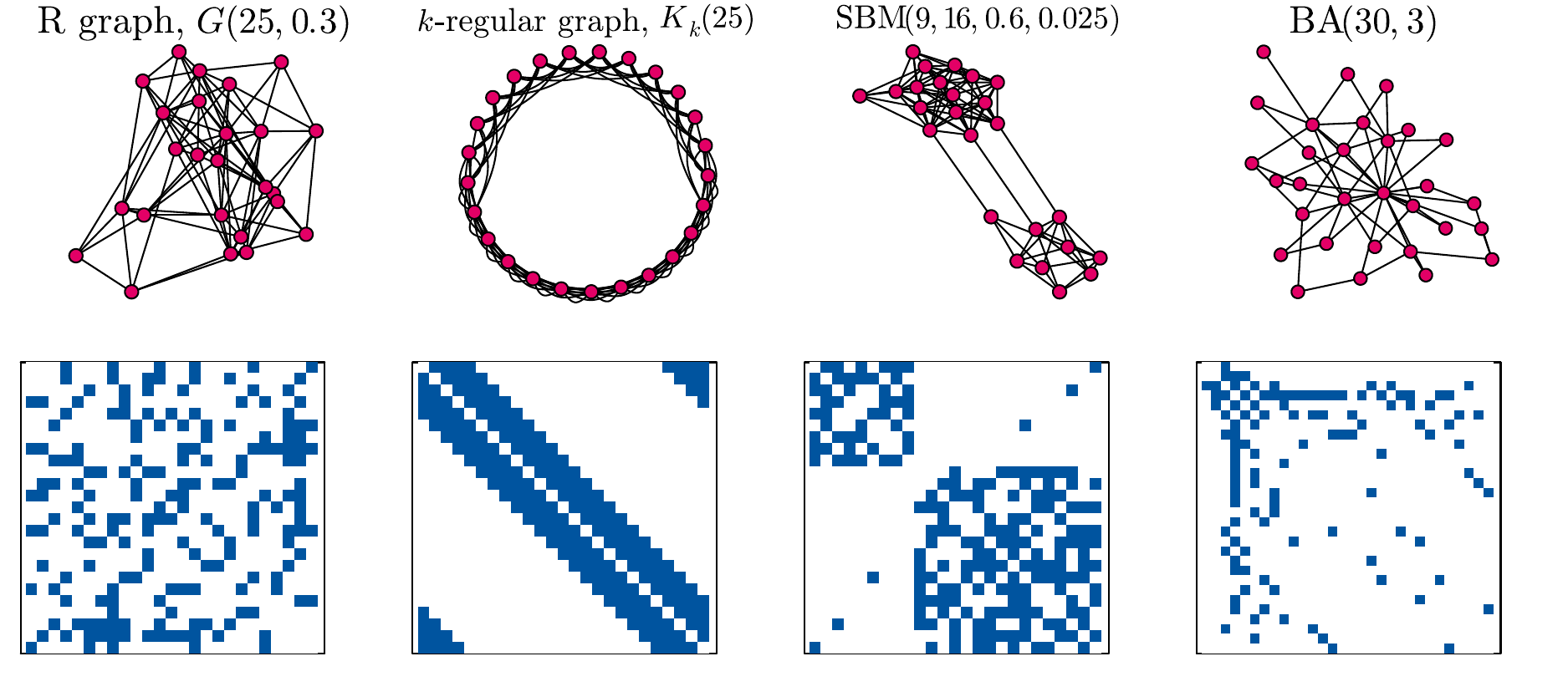}
\caption{
	Examples of graphs generated by \( G(n, p)\), \( K_k(n)\), \(\textrm{SBM}(n_1, n_2, p, q)\), and \(\textrm{BA}(n,m)\) respectively. Sparsity patterns of shift operators (adjacency matrices) are provided in the bottom row.
	\label{fig:models}
}
\end{figure}

\subsubsection{Singular value \( \sgs \left( XS^H \right)_{* \ii }\) }\label{subsec:sings}

Given the estimate of Theorem~\ref{thm:sourav}, the convergence rate is determined by the initialization \( \beta \) (i.e. defined in Remark~\ref{rem:initialisation}) and the singular value \( \sgs \left( ( X S^H)_{*\ii} \right) \) which encodes the graph structure, the initial feature data, and the labeled set.

In Figure~\ref{fig:synthetic_sings}, we demonstrate the dependence of the singular value \( \sgs \left( ( X  S^H )_{*\ii} \right)\) on the size of the labeled data set \( \bar n = | \ii  | \) for the following choices of the shift operator \( S \):
\begin{center}
\begin{tabular}{rl} 
	Adj. : & Adjacency Matrix,\\
	S-L Adj. : & Self-Loop Adjacency Matrix,\\
	Nor. S-L Adj. : & Normalized Self-Loop Adjacency Matrix,\\
	L : & Laplacian Matrix,\\
	Nor. L : & Normalized Laplacian Matrix.
\end{tabular}
\end{center}
For each fixed \( \bar n \), we uniformly sample a number of labeled sets \( \ii  \) in order to account for possible specific labeled configurations in the graph topology. Note that all the models mainly maintain the same ordering of the singular values in terms of the shift operator, i.e. graph Laplacian generates the highest \( \sgs \left( ( X  S^H )_{*\ii} \right)\), normalized self-loop adjacency operator corresponds to the lower singular value, and so on. Additionally, the estimations for adjacency and self-loop adjacency matrices remain numerically close, whilst the position of the corresponding normalized Laplacian curve depends on the model (except BA model which distinguishes between all shift operators): in contrast to \( G(n, p )\) and SBM, \( k\)-nearest neighbors graphs tend to not set apart the former three shift operators in terms of \( \sgs \left( ( X  S^H )_{*\ii} \right)\).

\begin{figure}[hbtp]
\centering
\includegraphics[width = 1.0\columnwidth]{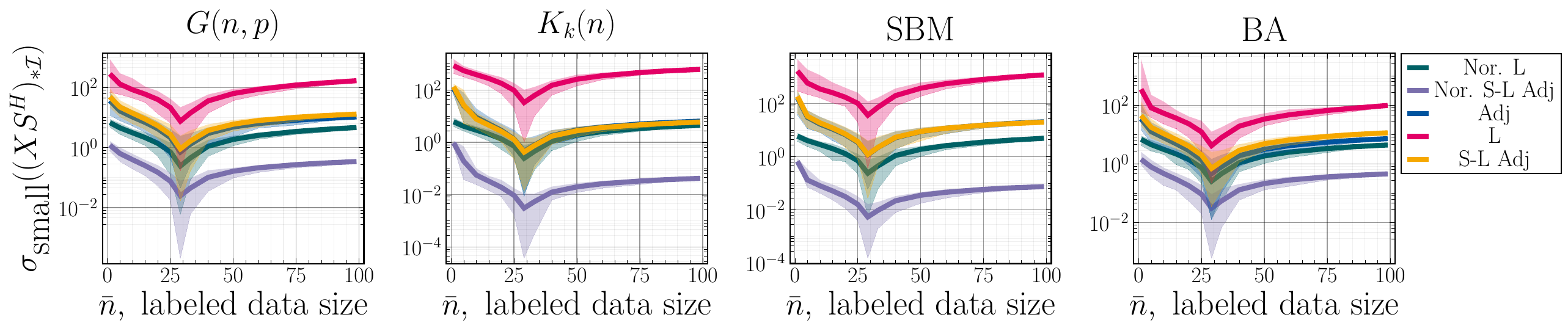}
\caption{
	\small
	Principle singular value \( \sgs \left( ( X  S^H )_{*\ii} \right)\) vs the size of labeled data \( \bar n \) for various shift operators \( S \) in models \( G(n, p)\), \( K_k(n)\), \(\textrm{SBM}(n_1, n_2, p, q)\), and \(\textrm{BA}(n,m)\),  respectively. Input feature dimension \( d_x = 30 \), network depth \( H = 2 \). Solid lines and semi-transparent areas correspond to the average value and spread of \( \sgs \left( ( X  S^H )_{*\ii} \right)\) for uniformly sampled sets \( \ii \). 
	\label{fig:synthetic_sings}
}
\end{figure}

At the same time, the singular value \( \sgs \left( ( X  S^H )_{*\ii} \right)\) for all models and all shift operators exhibits the same universal pattern of behaviour in terms of \( \bar n \): the monotonicity switch around \( \bar n = d_x \), for which one can justify with the following probabilistic argument. Let \( \Pi_{\ii }\) be a binary diagonal matrix such that \( \left[ \Pi_{\ii} \right]_{jj} = 1 \) if and only if \( j \in \ii \) and \( 0 \) otherwise; then, 
\begin{equation*}
\sgs \left( ( X  S^H )_{*\ii} \right) = \sgs \left(  X  S^H \Pi_{\ii} \right) = \sgs \left( \Pi_{\ii} S^H X^\top \right) = \min_{\substack{ \| z \| = 1 \\ z \perp \ker (\Pi_{\ii} S^H X^\top) }} \| \Pi_{\ii} S^H X^\top z \| .	
\end{equation*}
Assume the feature matrix \( X \) has a full-row rank, \( \text{rank}\,(X) = d_x < n \), and each \( x_i \perp \ker S \) (which is reasonable since any element in the null-space of $S$ is suppressed in the linear GNN \eqref{eqn:linear_GNN}), then the operator \( S^H X^\top \) has a full-column rank, \( \text{rank}\, ( S^H X^\top ) = d_x\), and a trivial kernel, so \( \ker (\Pi_{\ii} S^H X^\top) = \ker \Pi_{\ii} \cap \mathrm{im}\,(S^H X^\top)\). Note that \(\mathrm{dim}\, \ker \Pi_{\ii} = n - \bar n\) and \( \mathrm{dim}\, \mathrm{im}\, ( S^H X^\top) = d_x\), so these two subspaces intersect trivially in general position if \( d_x \le \bar n \). Hence,
\begin{equation*}
\mathbb E \sgs \left( ( X  S^H )_{*\ii} \right) = \mathbb E \min_{ \| z \| = 1 } \| \Pi_{\ii} S^H X^\top z \| = \frac{ \bar n }{ n } \sgs ( X S^H) 
\end{equation*}
where the expectation is taken for uniformly sampled labeled sets \( \ii \) of a fixed cardinality \( \bar n \) and \( \mathbb E \Pi_{\ii} = \frac{\bar n }{n } I \). Noticeably, each (\(d_x \le \bar n\))-part of the curves on Figure~\ref{fig:synthetic_sings} closely follows the average estimation \( \frac{ \bar n }{ n } \sgs ( X S^H) \). At the same time, left parts \( d_x > \bar n \) correspond to necessarily rank-deficient matrices \( ( X S^H )_{*\ii }\); since the permutation of matrix's columns conserves singular values \( \sigma_i \) and  each \( ( X S^H )_{*\ii }\) is a minor of the full matrix \( X S^H \) with permuted columns, smaller values of \( \bar n\) correspond to higher values of \(\sgs \left( ( X  S^H )_{*\ii} \right) \) due to the Cauchy's interlace theorem.

Finally, note that in the previous set of assumptions, \( \text{rank}\,(S^H X^\top) = \text{rank}\,(S X^\top) = d_x \), and \( \mathrm{im}\, (S X^\top )\) is spanned (in probabilistic terms, i.e. randomized matrix sketching) by first \( d_x \) dominant singular vectors. Then \(  \sgs ( X S^H)  =  \sgs ( S^H X^\top ) =  \sgs ( S^{H-1} (S X^\top )) = \min_{\| z \| =1 } \| S^{H-1} (S X^\top z) \| \approx \lambda_{d_x}^{H-1}  \sgs ( X S )  \) where \( \lambda_{d_x}\) denotes the \(d_x\)-th largest eigenvalues of the shift operator \( S \). As a result, the value of \( \lambda_{d_x} \) for different shift operators exponentially governs the convergence rate of the training in terms the networks depth \( H \) and affects the initialization, Remark~\ref{rem:initialisation}, and the gradient descent step, Theorem~\ref{thm:grad_desc}.

\subsubsection{Convergence of gradient flow training on the synthetic data}

Finally, we provide an illustration of the convergence of the gradient flow training (Theorem~\ref{thm:gradient_flow}) for all the aforementioned graph models with varying parameters. For each model, graphs on \( n = 200 \) are considered with \( 3 \) most common choices of shift operators: adjacency matrix, Laplacian and normalized Laplacian; input feature data \( X \) is sampled normally, \( d_x = 50 \). We consider a linear GNN architecture with \( H = 2 \) and hidden dimensions \( d_1 = d_2 = 32 \) (so \( W_1 \in \RR^{32 \times  50}\), \( W_2 \in \RR^{ 32 \times 32 }\), and \( W_3 \in \RR^{1 \times 32 }\) ); initialization follows Remark~\ref{rem:initialisation} with \( a = 2 \). 
The labeled data set \( \ii \) is composed of \( \bar n = 0.75 n \) and drawn uniformly; qualitative results below hold for all choices of \( \ii \). We provide the relative loss \( \frac{\ll(\W(T))- \tilde \ll_H}{ \ll(\W(0))- \tilde \ll_H } \) trajectory along the gradient flow and the singular value \( \sgs \left( (X S^H)_{*\ii} \right)\) for each considered setup; recalling Theorem~\ref{thm:gradient_flow}, the relative loss is expected to converge exponentially in time.

Along with the exponential convergence of the linear GNN, we demonstrate the following:
\begin{itemize}[topsep = 0pt, itemsep = 0pt]
\item the sparsity of the graph \( G \) typically affects the convergence rate with unnormalized shift operators such that sparser systems exhibit slower convergence;
\item GNNs using normalized Laplacians tend to be the least affected by the changes in the network structure with \( \textrm{BA}(n,m)\) model showing close-to-none dependency on the injected varying graph structure.
\end{itemize}

\paragraph{Erd\H{o}s--R\'enyi graph, \( G( n, p )\).}

In the classical \( G(n, p)\) model, the edge probability \( p \) regulates the sparsity pattern (as well as the connectivity and the information spread) of the system; we consider the range of \( p \) starting from a sparse but connected graph up to moderate values corresponding to the fast mixing, Figure~\ref{fig:gnp}. 

\begin{figure}[hbtp]
\centering
\includegraphics[width = 1.0\columnwidth]{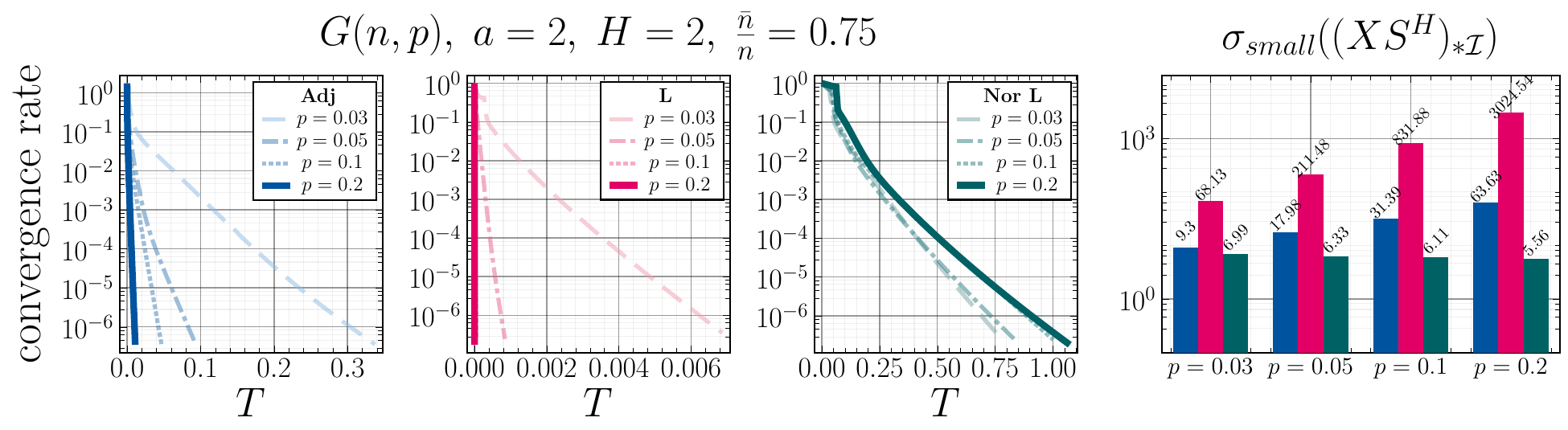}
\caption{
	\small
	Convergence rate of the relative loss \( \frac{\ll(\W(T))- \tilde \ll_H}{ \ll(\W(0))- \tilde \ll_H } \) in the gradient flow training for \( G(n, p)\) model, \( n = 200 \) and varying \( p \). Panes demonstrate loss flows for the different choices of the shift operator \( S \) (left to right: adjacency matrix, graph Laplacian, normalized Laplacian) and the singular values \( \sgs \left( (X S^H)_{*\ii} \right)\); for \( 4 \) different values of \( p \). 
	\label{fig:gnp}
}
\end{figure}

Note that in terms of the choice of shift operator, the convergence rate in Figure~\ref{fig:gnp} generally supports the ordering established in Figure~\ref{fig:synthetic_sings}: gradient flows with unnormalized Laplacian tend to converge overall faster than ones using adjacency matrix, whilst the case of normalized Laplacian remains remarkably stagnant. For both unnormalized operators, convergence rate tends to suffer from sparser graph structures with the effects noticeably worse for the graph Laplacian; on the contrary, in the case of the normalized Laplacian operator, denser graphs may result into the slower convergence rate supported by the smaller values of \( \sgs \left( (X S^H)_{*\ii} \right)\).

\begin{rem}[Normalized gradient flow for larger and denser graphs] 
	
	Linear GNN training in Figure~\ref{fig:gnp} is facilitated through the discrete integration of the gradient flow~\eqref{eqn:gradient_flow}. Since such integration is done in the direction of the anti-gradient, one aims to preserve the non-decreasing nature of the loss \( \ll (\W(T))\) along the trajectory by appropriately choosing the integration step which may be forced to be unfeasibly small. Consequently, the explosive growth of the singular value  \( \sgs \left( (X S^H)_{*\ii} \right)\) yields the faster convergence in time \( T \) (Theorem~\ref{thm:gradient_flow}), but, at the same time, it poorly affects the efficiency of the numerical integration due to the blow-up of gradient norms resulting into the smaller stepsizes, especially for larger and denser graphs with unnormalized shift operators. 

	One possible way to avoid the diminishing stepsizes is to consider instead the normalized gradient flow, \( \frac{dW_{\ell}(t)}{dt} = - \frac{ \nabla_{ W_{\ell} } \ll\big(\W(t)\big) } { \left\| \nabla_{ W_{\ell} } \ll\big(\W(t)\big) \right\|_F} \), allowing for a higher number of nodes and edges in graph \( G \); we provide example for \( G(n,p)\) model in the case of \( n = 500 \) in Figure~\ref{fig:gnp_norm}. Note that whilst the normalized flow is expected to converge to the same global optimum, one loses the guarantee of exponential convergence. Besides that, we posit that qualitative results and aspects of convergence remain the same for both normalized and unnormalized gradient flows, and further provide results only for the unnormalized integration to maintain the theoretical framework of Theorem~\ref{thm:gradient_flow}.
\begin{figure}[hbtp]
	\centering
	\includegraphics[width = 1.0\columnwidth]{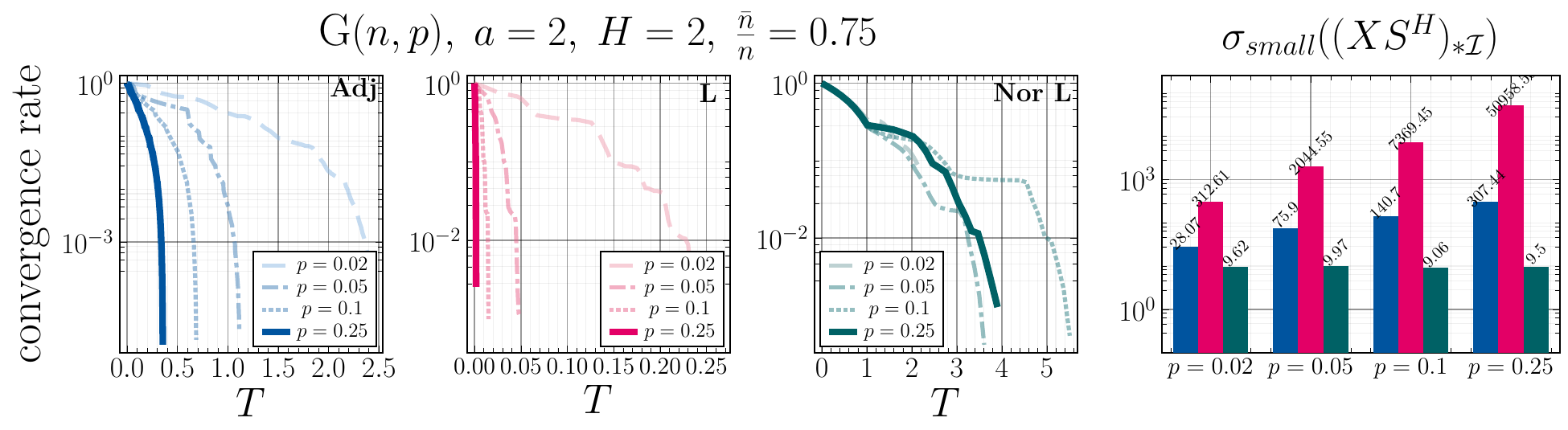}
	\caption{
		\small
		Convergence rate of the relative loss \( \frac{\ll(\W(T))- \tilde \ll_H}{ \ll(\W(0))- \tilde \ll_H } \) in the \textit{ normalized } gradient flow training for \( G(n, p)\) model, \( n = 500 \) and varying \( p \). Panes demonstrate loss flows for the different choices of the shift operator \( S \) (left to right: adjacency matrix, graph Laplacian, normalized Laplacian) and the singular values \( \sgs \left( (X S^H)_{*\ii} \right)\); for \( 4 \) different values of \( p \). 
		\label{fig:gnp_norm}
	}
\end{figure}
\end{rem}

\paragraph{\(k\)-nearest neighbors graph, \( K_k ( n ) \).}

In the \(k\)-regular structure of \( K_k( n )\), the convergence rate for the adjacency matrix is much closer to the case of the normalized graph Laplacian operator, Figure~\ref{fig:reg}. Moreover, the moderate values of the common degree \( k \) result into virtually the same loss trajectory during training, whilst the sparsest case is noticeably slower for unnormalized operators and faster for the normalized Laplacian.

\begin{figure}[h!]
\centering
\includegraphics[width = 1.0\columnwidth]{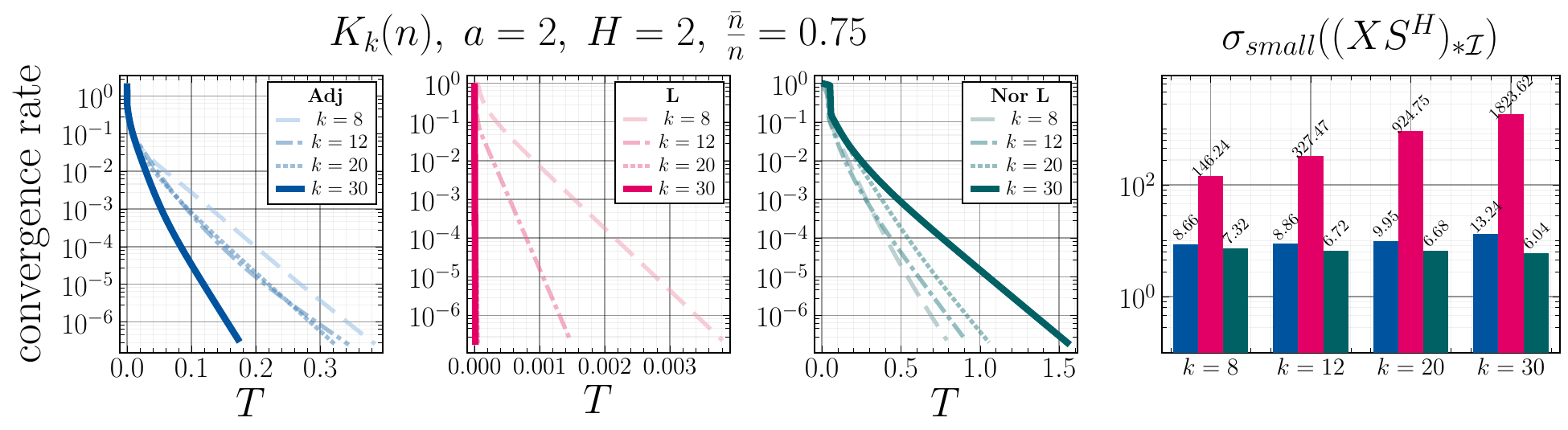}
\caption{
	\small
	Convergence rate of the relative loss \( \frac{\ll(\W(T))- \tilde \ll_H}{ \ll(\W(0))- \tilde \ll_H } \) in the gradient flow training for \( K_k(n)\) model, \( n = 500 \) and varying common degree \( k \). Panes demonstrate loss flows for the different choices of the shift operator \( S \)(left to right: adjacency matrix, graph Laplacian, normalized Laplacian) and the singular values \( \sgs \left( (X S^H)_{*\ii} \right)\); for \( 4 \) different values of \( k \). 
	\label{fig:reg}
}
\end{figure}

\paragraph{Stochastic Block Model, \( \textrm{SBM}(n_1, n_2, p, q)\).}

In comparison with the simple Erd\H{o}s--R\'enyi graph, we consider SBM model with \(n_1 : n_2 =  1 : 2 \) blocks, a fixed edge sparsity \( p \) inside each block  and the varying probability \( q \) corresponding to the inter-cluster edges (starting from the barely connected clusters up to an almost joined structure resembling \( G(n, p)\)), Figure~\ref{fig:sbm}. Noticeably, the singular value \( \sgs \left( (X S^H)_{*\ii} \right) \) is sufficiently less affected by the changes in \( q \) then by the changes in \( p \) in the case of \( G(n, p)\), Figure~\ref{fig:gnp}, with the adjacency matrix showing the largest distinction. As a result, the convergence rate for the SBM model is influenced by the inter-block connections during training only for larger \( q \) when the graph is well mixed for both adjacency matrix and unnormalized Laplacian.

\begin{figure}[h!]
\centering
\includegraphics[width = 1.0\columnwidth]{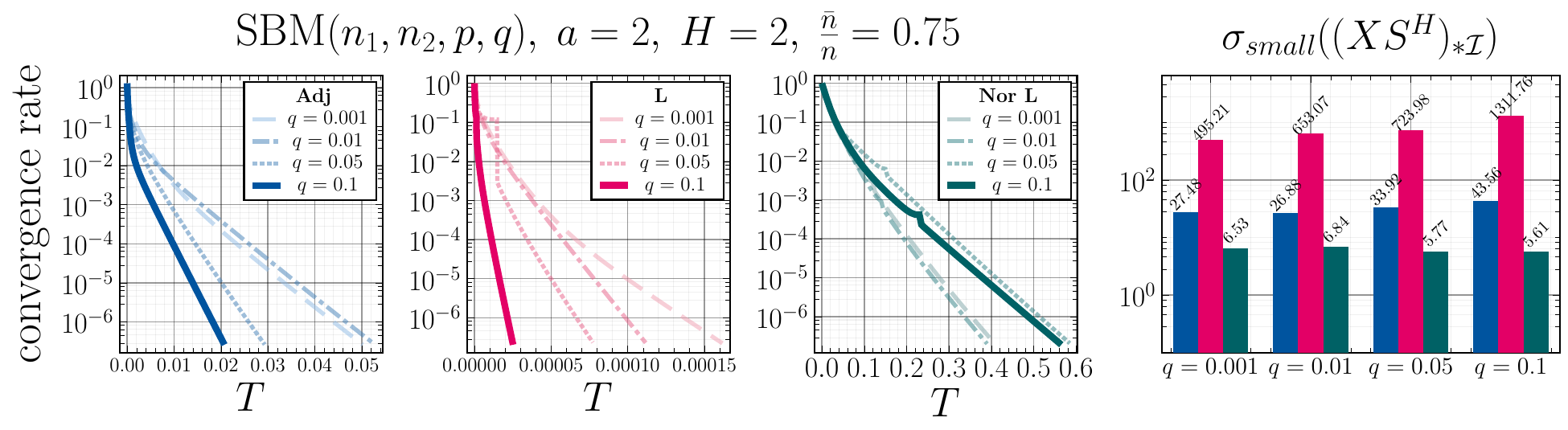}
\caption{
	\small
	Convergence rate of the relative loss \( \frac{\ll(\W(T))- \tilde \ll_H}{ \ll(\W(0))- \tilde \ll_H } \) in the gradient flow training for \( \textrm{SBM}(n_1, n_2, p, q)\) model, \( n_1 = 66 \), \( n_2 = 134\), \( p = 0.15 \) and varying \( q \). Panes demonstrate loss flows for the different choices of the shift operator \( S \)(left to right: adjacency matrix, graph Laplacian, normalized Laplacian) and the singular values \( \sgs \left( (X S^H)_{*\ii} \right)\); for \( 4 \) different values of \( q \). 
	\label{fig:sbm}
}
\end{figure}

\paragraph{Barab\'asi-Albert model, \( \textrm{BA}( n, m ) \).}

Finally, scale-free model \( \textrm{BA}( n, m ) \) noticeably breaks established pattern: in the case of small minimal degree values \( m \) (e.g. in the situations with many hanging nodes), the adjacency matrix exhibits the principle singular value \( \sgs \left( (X S^H)_{*\ii} \right)\) on par with the smallest one amongst chosen shift operators. Moreover, the overall convergence of the loss maintains the previously established pattern of dependency on the network's density for adajcency matrix and Laplacian operator, whilst the convergence for the normalized Laplacian is virtually independent on the changing network structure.

\begin{figure}[h!]
\centering
\includegraphics[width = 1.0\columnwidth]{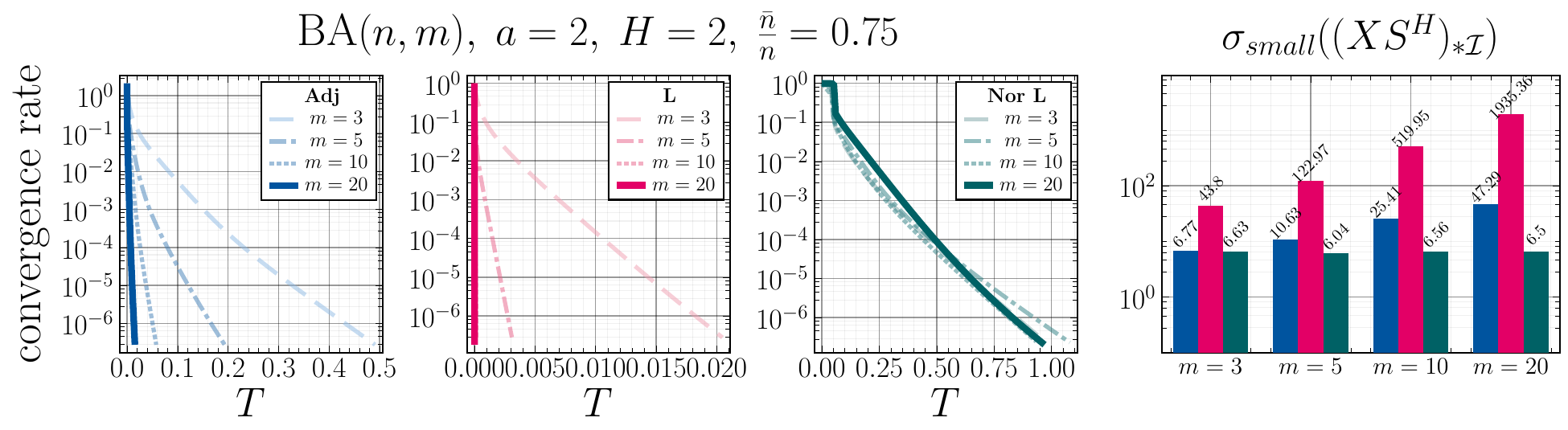}
\caption{
	\small
	Convergence rate of the relative loss \( \frac{\ll(\W(T))- \tilde \ll_H}{ \ll(\W(0))- \tilde \ll_H } \) in the gradient flow training for \( \textrm{BA}(n, m)\) model, \( n = 200 \) and varying minimal degree \( m \). Panes demonstrate loss flows for the different choices of the shift operator \( S \)(left to right: adjacency matrix, graph Laplacian, normalized Laplacian) and the singular values \( \sgs \left( (X S^H)_{*\ii} \right)\); for \( 4 \) different values of \( m \). 
	\label{fig:ba}
}
\end{figure}

\subsection{Convergence of gradient flow training on real-world graphs}
	
	In this subsection, we illustrate the gradient dynamics using a real-world graph dataset. We consider the graph structure with $3107$ nodes, corresponding to counties of the United States, and edges representing counties that share borders. The features data $X$, comprise land surface temperature, precipitation, sunlight, and fine particulate matter, and label data $Y$, represent air temperature, collected from CDC climate data for each month, averaged over 2008 \cite{CDCdataset}. The datasets are normalized to have zero mean and unit standard deviation.
	
	We now apply the gradient flow training on the graph neural network associated with the CDC climate data, and endorse Theorem \ref{thm:gradient_flow}. 
	In our numerical simulation, we assume that $75\%$ of the node label data are known and $\ii$ be the index set of such nodes. We consider a linear graph neural network defined in \eqref{eqn:linear_GNN} with $H=2$ and the initial weight matrices: $W_1\in \RR^{32\times 48}$ is a zero matrix, $W_2\in \RR^{32\times 32}$ is the identity matrix, and $W_3\in \RR^{12\times 32}$ has diagonal entries $1.5$ and all other $0$. Figure \ref{fig:three_figures} shows that the rate of convergence, i.e., the ratio of the relative MSE at time $T$ and the initial relative MSE, is accelerated by the smallest non-zero singular value of $(XS^H)_{*\ii}$. In the experiment, we consider different aggregate matrix $S$ of the graph as given in the subsection~\ref{subsec:sings}.
	
	
	The initial choice of weight matrices proportionate the rate of convergence of MSE to the global minimum. In particular, we consider a $2$-layer linear GNN with aggregation matrix $S$ as adjacency matrix of the graph, initial weight matrices $W_1\in \RR^{32\times 48}$ is a zero matrix and $W_2$ is the identity matrix of order $32$. Consequently, Figure \ref{fig:a_flow} illustrates that as the initial choice of diagonal entries of $W_3$ increases, so does the convergence rate. Although the convergence rate in gradient flow training is proportional to $\sgs\big((XS^H)_{*\ii}\big)$ and smallest singular value of initial weight matrices, graph neural networks are usually trained iteratively in practice. The iteration step must be sufficiently small depending on significant values of $\sgs\big((XS^H)_{*\ii}\big)$ and $\sgmin(W_i(0))$. 
	\begin{figure}[h!]
		\centering
		\graphicspath{ {Figure/} }
		\begin{subfigure}{0.33\textwidth}
			\centering
			\includegraphics[width=1\linewidth]{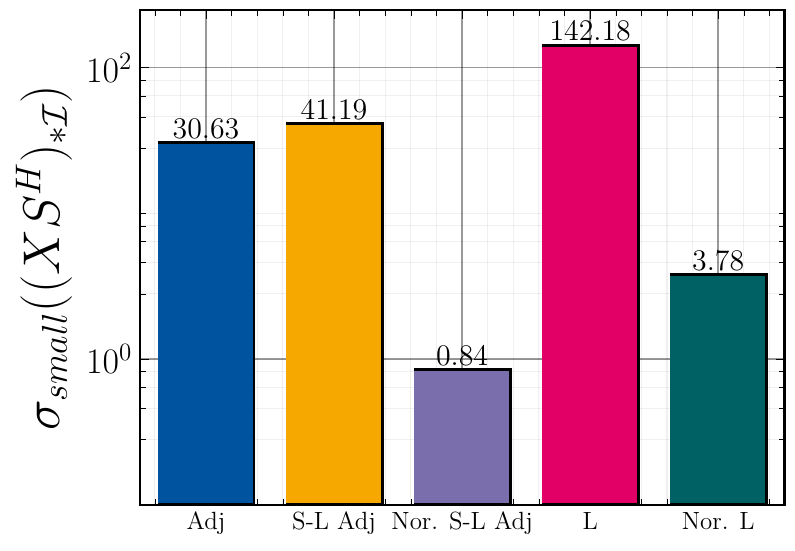}
			\captionsetup{font=tiny}
			\caption{$\sgs\big((XS^H)_{*\ii}\big)$}
		\end{subfigure}%
		\begin{subfigure}{0.33\textwidth}
			\centering
			\includegraphics[width=1\linewidth]{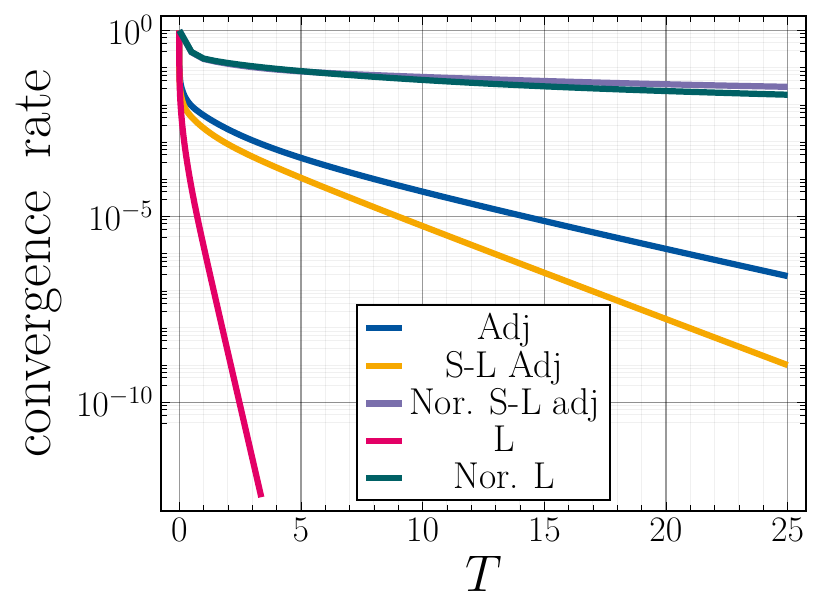}
			\captionsetup{font=tiny}
			\caption{Convergence rate vs Aggregation matrix}
		\end{subfigure}%
		\begin{subfigure}{0.33\textwidth}
			\centering
			\includegraphics[width=1\linewidth]{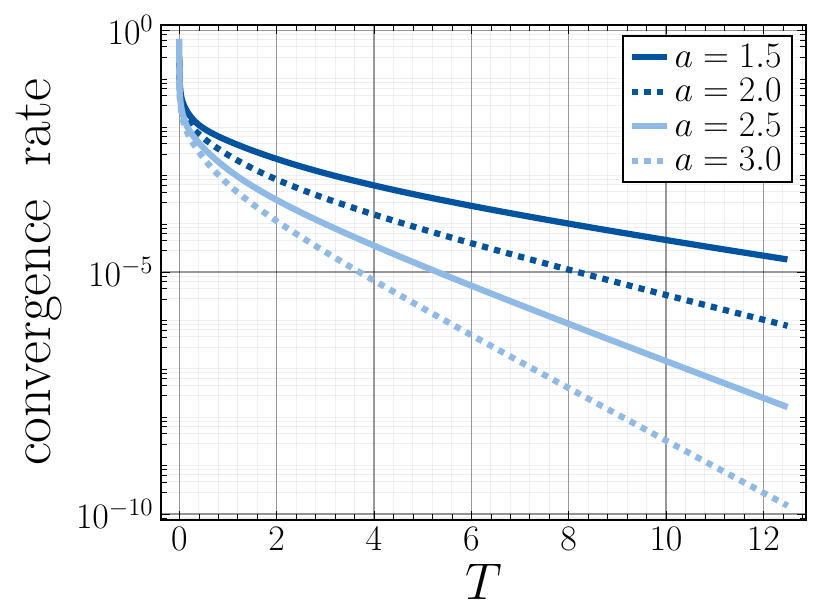}
			\captionsetup{font=tiny}
			\caption{Convergence rate vs Initial weight}
			\label{fig:a_flow}
		\end{subfigure}
		\caption{Comparison of convergence rate in GNNs training for different aggregation matrices. Left: Chart of smallest non-zero singular value of $(XS^H)_{*\ii}$ for different $S$. Middle: In gradient flow training of GNNs with fixed initial weight matrices, convergence rate is proportional to $\sgs\big((XS^H)_{*\ii}\big)$. Right: Convergence rate of a fixed aggregation matrix depends on initial weight matrices.}
		\label{fig:three_figures}
	\end{figure}

		\section{Conclusion}
		We study the convergence of gradient dynamics in linear graph neural networks. In order to avoid the existence of any sub-local minimum, we focus our study on GNNs without bottleneck layer. We demonstrate that, with appropriate initialization, the square loss converges to the global minimum at an exponential rate during gradient flow training of linear GNNs. Furthermore, a balanced initialization minimizes the total energy of the weight parameters at the global minimum. The convergence rate depends on the aggregation matrix and the initial weights, as confirmed through numerical experiments on synthetic and real-world datasets.

	\section*{Acknowledgment}
	The authors acknowledge funding by the Deutsche Forschungsgemeinschaft (DFG, German Research Foundation) - Project number 442047500 through the Collaborative Research Center ``Sparsity and Singular Structures” (SFB 1481).
	
	\bibliographystyle{abbrv}
	\bibliography{A07_paper1}
\end{document}